\documentclass[11pt]{article}

\title{Rate-Optimal Policy Optimization \\ for Linear Markov Decision Processes} 

\author{
Uri Sherman%
\thanks{Blavatnik School of Computer Science, Tel Aviv University; \texttt{urisherman@mail.tau.ac.il}.
}
\and
Alon Cohen%
\thanks{School of Electrical Engineering, Tel Aviv University and Google Research; \texttt{alonco@tauex.tau.ac.il}.
}
\and
Tomer Koren%
\thanks{Blavatnik School of Computer Science, Tel Aviv University and Google Research; \texttt{tkoren@tauex.tau.ac.il}.}
\and
Yishay Mansour%
\thanks{Blavatnik School of Computer Science, Tel Aviv University and Google Research; \texttt{mansour.yishay@gmail.com}. 
}
}

\usepackage{header_arxiv}
\usepackage{header}

\begin{document}

\maketitle

% %%%%%%%%%%%%%%%%%%%
% %%%% END ARXIV %%%%
% %%%%%%%%%%%%%%%%%%%

\newcommand{\us}[1]{\textcolor{purple}{\bfseries\{US: #1\}}}
\newcommand{\tk}[1]{\textcolor{magenta}{\bfseries\{TK: #1\}}}
\newcommand{\ym}[1]{\textcolor{red}{\bfseries\{YM: #1\}}}

\newcommand{\wtO}{\smash{\widetilde O}}
\newcommand{\temp}[1]{\textcolor{teal}{\{#1\}}}

% \begin{center}
% \textbf{\Large Draft: Please do not distribute}    
% \end{center}

\begin{abstract}
    We study regret minimization in online episodic linear Markov Decision Processes, 
    and propose a policy optimization algorithm that is computationally efficient, and obtains rate optimal $\wtO (\sqrt K)$ regret where $K$ denotes the number of episodes.
    Our work is the first to establish the optimal rate (in terms of~$K$) of convergence in the stochastic setting with bandit feedback using a policy optimization based approach, and the first to establish the optimal rate in the adversarial setup with full information feedback, for which no algorithm with an optimal rate guarantee was previously known.
\end{abstract}

\section{Introduction}
Policy Optimization (PO) algorithms are a class of methods in Reinforcement Learning (RL; \citealp{sutton2018reinforcement,mmt2022rlbook}) where the agent's policy is iteratively updated according to the (possibly preconditioned) gradient of the value function w.r.t.~policy parameters. 
From a theoretical perspective, framing the optimization process as one that follows Mirror Descent \cite{nemirovskij1983problem,beck2003mirror} updates leads to strong online guarantees that go beyond stationary or stochastic rewards, and apply more generally for \emph{any} (possibly adversarial) reward sequence \cite{shani2020optimistic,luo2021policy}. 
Furthermore, PO methods are easy to implement in practice and (perhaps, one could say, somewhat in line with theory) exhibit favorable robustness properties when applied to real world problems ranging from robotics \cite{levine2013guided,schulman2015high, haarnoja2018soft}, computer games \cite{schulman2017proximal}, and more recently training of large language models \cite{ouyang2022training}.

Notwithstanding their popularity and theoretical appeal, current results \cite{agarwal2020pc,zanette2021cautiously,liu2023optimistic,zhong2023theoretical} in the function approximation setting with linear MDP \cite{jin2020provably} assumptions fall short of establishing the optimal dependence on the number of episodes $K$; arguably, the most important problem parameter.

In this work, we establish that an optimistic variant of the classic natural policy gradient
\footnote{To be precise, our algorithm is the classic NPG with softmax parametrization equipped with an optimistic linear function approximation routine for action-value estimates.}
(NPG; \citealp{kakade2001natural}) obtains the optimal (up to logarithmic factors) $\wtO (\sqrt K)$ regret when combined with a short reward free warmup period and a suitable bonus update schedule.
Our results hold for adversarial losses when the learner is given full information feedback, and for stochastic losses  when given bandit feedback. Thus our algorithm is also the first (and currently, the only) method to obtain rate optimal regret (be it by PO or any other approach) for adversarial losses with full feedback in the linear MDP setup.

\subsection{Summary of contributions}

We consider online learning in a finite horizon episodic linear MDP, where an agent interacts with the environment over the course of $K$ episodes.
In each episode $k\in [K]$, the agent interacts with the MDP $\cM_k=(\cS, \cA, H, \cbr{\P_h}, \cbr{\l_h^k}, s_1)$, that shares all elements with MDPs of other episodes except for the loss functions.
Our central structural assumption is that the dynamics and losses are \emph{linear}; that there exist feature embeddings $\phi, \psi_1, \ldots, \psi_H$ such that $\P_h(s' | s, a) = \phi(s, a)\T \psi_h(s')$, and $\E[\l_h^k(s, a)|s,a] = \phi(s, a)\T g_{h,k}$, for some $g_{h,k} \in \R^d$.
The objective of the agent is to minimize her \emph{regret}, defined by the sum of value functions (namely, the expected cumulative loss) of the agent minus the sum of values of the best policy in hindsight.
% \begin{aligni*}
%     \min_{\pi^1, \ldots, \pi^K}
%     \cbr{\sum_{k=1}^K V^{k, \pi^k} - V^{k, \pi^\star}}.
% \end{aligni*}

Our main contribution in this paper is a computationally efficient policy optimization algorithm (see~\cref{alg:oppo_linmdp}), that guarantees an $\wtO (\sqrt K)$ regret bound under either of the following two conditions:
\begin{itemize}[leftmargin=*]
    \item For any (possibly adversarial) loss sequence $\cbr{g_{h,k}}$, when given \emph{full feedback}, meaning the agent observes $g_{1,k}, \ldots, g_{H,k}$ after each episode $k$.
    \item For stationary losses, namely $g_{h,k} = g_h \forall k$, when given noisy \emph{bandit feedback}, meaning the agent observes only $l_h^k \eqq \l_h^k(s_h^k, a_h^k)$, and it holds that
    $l_h^k\in[-1, 1]$ and that the expected value 
    of $l_h^k$ conditioned on past interactions is
    $\phi(s_h^k, a_h^k)\T g_h$.
\end{itemize}

\subsection{Overview of techniques}
The difficulty encountered in recent attempts (\citealp{liu2023optimistic,zhong2023theoretical}, and to an extent also in \citealp{sherman2023improved}) towards establishing the rate optimal $\sqrt K$ stems from the need to control the capacity of the policy class explored by the optimization process. Since the dynamics in linear MDPs cannot be estimated pointwise, the estimation procedure of the action-value function involves a linear regression sub-routine where the dependent variable is given by the value function estimate from the previous timestep, which depends on past rollouts in a way that breaks the martingale structure. Thus, to establish concentration, an additional uniform convergence argument is required in which the capacity of the policy class plays a central role.

To illustrate,
let us consider a simplified, non-optimistic estimation routine with non-zero immediate losses only at step $H$, and let $\cbr{(s_h^i, a_h^i, s_{h+1}^i)}_{i=1}^{k-1}$ denote a dataset of past agent transitions, and $\widehat V^{k}_{h+1}$ the value function estimated in step $h+1$. Then the estimation step on time $h$ is given by:
\begin{align*}
    \widehat v_h^k 
    &= \argmin_{v \in \R^d}\cbr{
        \sum_{i=1}^{k-1} \br{
	   \phi(s_h^i, a_h^i)\T v -  \widehat V^{k}_{h+1} (s_{h+1}^i)
        }^2
        },
	\\
	\widehat Q^k_h(s, a) 
		&=
		{\rm truncate}\sbr{\widehat \P_h^k \widehat V_{h+1}^k (s, a) \eqq \phi(s, a)\T \widehat v_h^k},
\end{align*}	
where $\rm truncate[\cdot]$ denotes some form of clipping used to keep the estimated action-values in reasonable range (e.g., $[-H, H]$). Notably, $\widehat V_{h+1}^k$ was itself estimated using the same procedure in the previous backward induction step, combined with an expectation given by the agent's policy:
\begin{align*}
	\widehat V^k_{h+1}(s) = \abr{\pi_{h+1}^k(\cdot|s), \widehat Q^k_{h+1}(s, \cdot)},
\end{align*}
which means the estimated quantity is a random variable that depends on \emph{all} past trajectories through the agent's policy.
Hence, to establish a least squares concentration bound, the common technique (originally proposed in this context in the work of \citealp{jin2020provably}) dictates arguing uniform convergence over the class of all possible value functions $\widehat V_{h+1}^k$ explored by the learner. Further, the capacity of the class of learner value functions is inevitably tied to the capacity of the learner's policies, and when employing mirror descent updates, these are parameterized by the sum of past action-value functions:
\begin{align*}
	\pi_{h+1}^k(a|s) \propto \exp\br{-\eta \sum_{i=1}^{k-1}\widehat Q_{h+1}^i(s, a)}.
\end{align*}
Now, the problem is that the truncation of the Q-functions implies the above expression does not admit a low dimensional (independent of $k$) representation, and thus leads to the agent's policy and value classes having prohibitively large covering number. 

The main component of our approach is to employ a reward free warmup period, that eventually allows to forgo the truncation of the action value function, thereby reducing the policy class capacity.
Indeed, if the action-value functions were not truncated, the policy parameterization could be made effectively independent  (up to log factors) of $k$, as the sum of Q-functions will ``collapse'' into a single $d$ dimensional parameter of larger norm:
\begin{align*}
	\pi_{h+1}^k(a|s) 
        \propto 
	% \exp\br{-\eta \sum_{i=1}^{k-1}\widehat Q_{h+1}^i(s, a)}
	% = 
	\exp\br{\phi(s, a)\T \theta_{h+1}^k}
	,
\end{align*}
where $\theta_{h+1}^k = -\eta \sum_{i=1}^{k-1} \widehat v_{h+1}^i$.
In order to remove the truncations, we observe they are actively involved only in those regions of the state space that are poorly explored; indeed, assume the least squares errors are boudned as:
\begin{align*}
	\av{\widehat \P_h^k \widehat V_{h+1}^k(s, a)
	-
	\P_h \widehat V_{h+1}^k(s, a)}
%	=
%	\phi(s, a)\T (\widehat v_{h+1}^k - v_{h+1}^k)
	\leq \beta \norm{\phi(s, a)}_{\Lambda_{k, h}^{-1}},
\end{align*}
where $\Lambda_{k, h} \eqq I + \sum_i \phi(s_h^i, a_h^i)\phi(s_h^i, a_h^i)\T$ for some $\beta$ that depends (among other quantities) on $\max_{s'}\widehat V_{h+1}^k(s')$, and assume we have already shown that $\widehat V_{h+1}^k(s')\lesssim H$ for all $s'$.
Then as long as $\phi(s, a)$ points in a well explored direction in the state-action space --- concretely one where $\norm{\phi(s, a)}_{\Lambda_{k, h}^{-1}} \leq 1/(\beta H)$ --- we will get that:
\begin{align*}
	\widehat Q^k_h(s, a) &= \P_h \widehat V_{h+1}^k(s, a) \pm \frac{1}{H}
    \\
    \implies 
    &\av{\widehat Q^k_h(s, a)} \leq \av{\P_h \widehat V_{h+1}^k(s, a)} + \frac{1}{H}
    \lesssim H + \frac1H
	.
\end{align*}
Thus, forgoing truncations and if all directions were well explored, we would 
get $\norm[b]{\widehat V^k_h}_\infty \leq 	\norm[b]{\widehat V^k_{h+1}}_\infty + \frac1H$, and continuing inductively we accumulate errors across the horizon in an additive manner; $\norm[b]{\widehat V^k_h}_\infty \lesssim H + (H-h)/H$.
Now, while we cannot ensure sufficient exploration in \emph{all} directions, we can in fact ensure it  in ``most'' directions (those which are reachable w.p.~$\gtrsim 1/\sqrt K$) using a properly tuned reward free warmup phase, which is based on the algorithm developed in \citet{wagenmaker2022reward}.
The technical argument roughly follows the above intuition, carefully controlling the least squares errors through an inductive argument. This way, we establish the estimated value functions remain in the low capacity function class, for which we have a suitable uniform concentration bound.

\subsection{Additional related work}

\paragraph{Linear MDPs with adversarial costs.}
Most relevant to our paper is the recent work of \citet{zhong2023theoretical}, who consider the same adversarial setup as ours and establish a $\widetilde O(K^{3/4})$ regret bound, using an optimistic policy optimization framework similar to ours, but with an additional batching mechanism.
Several recent papers consider the more general setting consisting of \emph{adversarial costs and bandit-feedback}.
\citet{neu2021online} obtain a rate optimal regret bound assuming \emph{known dynamics} and a certain exploratory condition.
In the general setting without additional assumptions, \citet{luo2021policy} was the first to establish a sublinear regret bound. The followup works of
\citet{dai2023refined,sherman2023improved} obtain respectively, $\widetilde O(K^{8/9}), \widetilde O(K^{6/7})$ regret, and \citet{kong2023improved} obtain $\widetilde O(K^{4/5} + 1/\lambda_{\min}^\star)$ (here, $\lambda_{\min}^\star$ denotes the minimum eigenvalue of the best exploratory policy's 2nd moment matrix) albeit with a computationally inefficient algorithm.
Finally, a very recent preprint~\citep{liu2023towards} 
establishes the current state-of-the-art results for this setting; $\widetilde O (K^{3/4})$ with a computationally efficient algorithm, and $\widetilde O(\sqrt K)$ with a computationally inefficient one.%\ym{We should stress that they are later, maybe say: "Recently, and based on our work", however I am worried about breaking the anonymity}

\paragraph{Policy optimization in tabular and linear MDPs.}
Most of the currently published works that consider policy optimization algorithms in the learning setup that necessitates exploration were mentioned in the introduction. In particular, the work of \citet{liu2023optimistic} considers the same stochastic setup as ours and obtains a $\widetilde O (1/\epsilon^3)$ sample complexity for a different variant of the optimistic NPG algorithm.
Many recent works \citep[e.g., ][]{bhandari2019global,liu2019neural, agarwal2021theory,lan2022policy,xiao2022convergence, yuan2022linear} study convergence properties of policy optimization methods from a pure optimization perspective or subject to exploratory assumptions; in this setup, exploration need not be handled algorithmically, and rates much faster than $O(\sqrt K)$ regret are achievable when access to exact value function gradients is granted.

\paragraph{RL with function approximation}
The study of MDPs with linear structure in the form we adopt here was initiated with the works of \citet{yang2019sample,yang2020reinforcement,jin2020provably}, and has lead to an abundance of papers considering algorithmic approaches to various problem setups \citep[e.g., ][]{zanette2020learning,wei2021learning,wagenmaker2022reward}.
The linear mixture MDP \cite{modi2020sample,ayoub2020model,zhou2021nearly,zhou2021provably} is a different model that in general is incomparable with the linear MDP \cite{zhou2021provably}.
There is also a rich line of works studying statistical properties of RL with more general function approximation \citep[e.g., ][]{jiang2017contextual,jin2021bellman,du2021bilinear}, although these usually do not provide computationally efficient algorithms.

\section{Preliminaries}

\paragraph{Episodic MDPs.}
A finite horizon episodic MDP is defined by the tuple
$\cM = \br{\cS , \cA, H, \P , \l, s_1}$,
where $\cS$ denotes the state space, $\cA$ the action set, $H \in \mathbb Z_+$ the length of the horizon,
$\P = \cbr[s]{\P_h}_{h\in[H]}$ the time dependent transition function, $\l = \cbr{\l_h}_{h\in[H]}$ a sequence of loss functions, and $s_1 \in \cS$ the initial state that we assume to be fixed w.l.o.g.
The transition density given the agent is at state $s\in \cS$ at time $h$ and takes action $a$ is given by $\P_h(\cdot | s, a) \in \Delta(\cS)$. After the agent takes an action on the last time step $H$, she transitions to a fixed terminal state $s_{H+1} \in \cS$ and the episode terminates immediately.
We assume the state  space $\cS$ is a (possibly infinite) measurable space, and that the action set $\cA$ is finite with $A \eqq |\cA|$. 
A policy is defined by a mapping $\pi\colon \cS \times [H] \to \Delta(\cA)$, where $\Delta(\cA)$ denotes the probability simplex over the action set $\cA$. We let $\pi_h(\cdot|s) \in \Delta(\cA)$ denote the distribution over actions given by $\pi$ at $s, h$.
Finally, we use the convention that for any function $V \colon \cS \to \R$, we interpret $\P_h V \colon \cS \times \cA \to \R$ as the result of applying the conditional expectation operator $\P_h$; $\P_h V(s, a) \eqq \E_{s' \sim \P_h(\cdot |s, a)} V(s')$.

\paragraph{Episodic Linear MDPs.}
Our central structural assumption is that the learner interacts with a \emph{linear MDP} \cite{jin2020provably}, defined next.
\begin{definition}[Linear MDP]
\label{def:linmdp}
    An MDP $\cM = (\cS, \cA, H$, $\P, \l, s_1)$ is a linear MDP if the following holds.
    There is a feature mapping $\phi\colon \cS \times \cA \to \R^d$ that is \textbf{known} to the learner, and $H$ signed vector-valued measures $\psi_h \colon \cS \to \R^d$ that are \textbf{unknown},
    such that for all $h, s, a, s' \in [H-1] \times \cS \times \cA \times \cS$:
    \begin{align}
		\P_h(s' | s, a)	&= \phi(s, a)\T \psi_h(s').
    \end{align}
    W.l.o.g., we assume $\norm{\phi(s, a)}\leq 1$ for all $s, a$,
    and that for any measurable function $f\colon \cS \to \R$ with $\norm{f}_\infty \leq 1$, it holds that $\norm{\int \psi_h(s') f(s') {\rm d} s'} \leq \sqrt d$ for all $h\in [H]$.
    In addition, for all $s,a,h$:
    \begin{align}
	\l_h(s, a) &= \phi(s, a)\T g_{h},
    \end{align}
    where $\cbr{g_{h}}\subset \R^d$. W.l.o.g., we assume $\av{\phi(s, a)\T g_{h}} \leq 1$ for all $s, a, h$, and
	$\norm{g_{h}} \leq \sqrt d$ for all $h$.
\end{definition}

\paragraph{Problem setup.}

We consider linear MDPs in two setups; \emph{adversarial} and \emph{stochastic}. In the adversarial setup defined formally next, we assume the agent interacts with a sequence of $K\geq 1$ MDPs over the course of $K$ episodes that share all elements other than the loss functions, which may change adversarially. 
\begin{assumption}[Adversarial Linear MDP with full-feedback]
\label{assume:linmdp_adv}
    The learner interacts with a sequence of MDPs $\cbr[b]{\cM^k}_{k=1}^K$, $\cM^k = (\cS, \cA, H, \P, \l^k, s_1)$ that share all elements other than the loss functions. Each MDP $\cM^k$ is a linear MDP as per \cref{def:linmdp}.
The feedback provided to the learner on episode $k$ time step $h$ is the low dimensional cost vector $g_{k, h} \in \R^d$, where $g_k = \br{g_{k,1}, \ldots, g_{k, H}} \in \R^{dH}$ is the $d$ dimensional representation of $\l^k = \br{\l_1^k, \ldots, \l_H^k}$. 
\end{assumption}

In the stochastic setup, we assume the agent interacts with a single linear MDP over the course of $K \geq 1$ episodes, and receives only noisy \emph{bandit}-feedback. 
\begin{assumption}[Stochastic Linear MDP with bandit-feedback]
\label{assume:linmdp_stochastic}
    In each episode, the learner interacts with the same linear MDP $\cM = (\cS, \cA, H, \P, \l, s_1)$.
The feedback provided to the learner on episode $k$ time step $h$ is the random instantaneous loss $l_h^k \eqq \l_h^k(s_h^k, a_h^k)$, where $s_h^k, a_h^k$ denote the state and action visited by the agent on episode $k$ time step $h$. It holds that $\E\sbr{l_h^k \mid s_h^k, a_h^k, \br{l_h^{k'}, s_h^{k'}, a_h^{k'}}_{k'<k}} = \l_h(s_h^k, a_h^k)$, and $\av{\l_h^k(s_h^k, a_h^k)}\leq 1$ almost surely.
\end{assumption}

The pseudocode for learner environment interaction, encompassing both assumptions is provided below in Protocol~\ref{prot:learner_env_interaction}.
We make the following final notes with regards to the model we consider: 
(1) for any $s,a \in \cS \times \cA$, the agent may evaluate $\phi(s, a)$ in $O(1)$ time;
(2) In the adversarial setup, we assume an oblivious and deterministic adversary. Specifically, that the sequence of loss functions is chosen in advance, before interaction begins. 
\begin{protocol}[!ht]
    \caption{Learner-Environment Interaction}
    \label{prot:learner_env_interaction}
	\begin{algorithmic}
            \STATE parameters: $(\cS, \cA, H, \P, \phi, s_1; K)$ 
            % \STATE assume: $\P, \phi$ form linear dynamics as per \cref{def:linmdp}.
            \STATE Nature chooses $\begin{cases}
                        \text{\emph{Adv.:}}
                        \quad &\text{$\cbr{g_{k}}_{k=1}^K \in \R^{dH}$;}
                        \\
                        \text{\emph{Stoch.:}}
                        \quad &\text{$g \in \R^{dH}$, }
                        \text{and sets $g_k\equiv g\;\forall k$}
                    \end{cases}$
            % \STATE Nature chooses $K$ cost vectors $\cbr{g_{k}} \in \R^{dH}$
            % (In the \emph{stochastic case}, $g_k \equiv g_1$ for all $k$)
	    \FOR{$k=1, \ldots, K$}
                \STATE agent decides on a policy $\pi^k$
                % \STATE adversary observes $\pi^k$

%                \STATE define $\l_h^k \colon \cS \times \cA \to \R$ by $\l_h^k(s, a) = \phi(s, a)\T g_h^k$.

	    	\STATE environment resets to $s_1^k = s_1$
		    \FOR{$h=1, \ldots, H$}
			    	\STATE agent observes $s_h^k \in \cS$
			    	\STATE agent chooses $a_h^k \sim \pi^k_h(\cdot|s_h^k)$
			    	\STATE agent incurs loss 
			    		$\phi(s_h^k, a_h^k)\T g_{k,h}$ 
                    \STATE agent observes $\begin{cases}
                        \text{\emph{Full-feedback:}}
                        \quad &\text{$g_{k, h}$}
                        \\
                        \text{\emph{Bandit-feedback:}}
                        \quad &\text{ $\l_h^k(s_h^k, a_h^k)$}
                    \end{cases}$
                    \STATE environment transitions to $s_{h+1}^k \sim \P_h(\cdot|s, a)$
	        \ENDFOR
        \ENDFOR
	\end{algorithmic}
\end{protocol}

\paragraph{Learning objective.}
The expected loss of a policy $\pi$ when starting from state $s\in \cS$ at time step $h\in [H]$ is given by the value function;
\begin{align}
	V_h^{\pi}(s; \l) 
	\eqq 
	\E \sbr{\sum_{t=h}^H \l_t(s_t, a_t) \mid s_h = s, \pi, \l},
	\label{eq:V_def}
\end{align}
where we use the extra $(;\l)$ notation to emphasize the specific loss function considered. 
The expected loss conditioned on the agent taking action $a\in \cA$ on time step $h$ at $s$ and then continuing with $\pi$ is given by the action-value function;
\begin{align}
	Q_h^{\pi}(s, a; \l) 
	\eqq 
	\E \!\sbr{\sum_{t=h}^H \l_t(s_t, a_t) \mid s_h = s, a_h = a, \pi, \l}\!
	.
	\label{eq:Q_def}
\end{align}

The value and action-value functions of a policy $\pi$ in the MDP $\br{\cS, \cA, H, \P, \l^k, s_1}$ associated with episode $k\in [K]$ are denoted by, respectively;
\begin{align}
	V_h^{k, \pi}(s) 
	\eqq 
	V_h^{\pi}(s; \l^k)
	;\; 
	Q_h^{k, \pi}(s, a) 
	\eqq 
	Q_h^{\pi}(s, a; \l^k), 
\end{align}
where $V_h^{\pi}(s; \l^k)$ and $Q_h^{\pi}(s, a; \l^k)$ have been defined in \cref{eq:V_def,eq:Q_def}.
For the sake of conciseness, we further define
\begin{align*}
	V^{k, \pi} \eqq V_1^{k, \pi}(s_1) 
\end{align*}
We let $\pi^\star$ denote the best policy in hindsight;
\begin{align*}
	\pi^\star 
	% \eqq
	% \pi^\star \mid [\pi^1, \ldots, \pi^K]
	\eqq \argmin_{\pi} \cbr{
	\sum_{k=1}^K 
		V_1^{k, \pi}(s_1)
	},
\end{align*}
and seek to minimize the \emph{pseudo regret} of the agent policy sequence $\pi^1, \ldots, \pi^K$;
\begin{align}
	\Reg \eqq \sum_{k=1}^K 
	V^{k, \pi^k} - V^{k, \pi^\star}.
\end{align}

\paragraph{Occupancy measures.}
We denote the occupancy measure of a policy $\pi$ by
\begin{align}
	\mu_h^{\pi}(s, a)
	&\eqq \Pr\br{s_h = s, a_h = a \mid \pi},
\end{align}
and additionally denote $\mu^k_h \eqq \mu_h^{\pi^k}$, and $\mu_h^\star \eqq \mu_h^{\pi^\star}$.

\paragraph{Additional notation. }

We let $\norm{\cdot} = \norm{\cdot}_2$ denote the standard Euclidean norm, and
for a positive definite matrix $\Lambda\in \R^{d\times d}$, we let $\norm{v}_\Lambda = \sqrt {v\T \Lambda v}$ denote the weighted norm induced by $\Lambda$.
Further, we let $\norm{\Lambda} = \norm{\Lambda}_{\rm op} = \max_{v, \norm{v}=1} v\T \Lambda v$ denote the operator norm of $\Lambda$.
%Finally, we use $\rm{clip}\sbr{x}^a_b \eqq \max\cb{\min\cb{x, a}, b}$ to denote clipping of a real scalar $x$ between $a\in \R$ and $b\in \R$.

\section{Algorithm and Main Result}

In this section, we present \cref{alg:oppo_linmdp} and our main theorem providing its regret guarantees.
At a high level, \cref{alg:oppo_linmdp} follows an optimistic policy optimization paradigm similar to \citet{shani2020optimistic} in the tabular case and more recently \citet{liu2023optimistic,zhong2023theoretical} in the linear MDP case.
The important difference is the utilization of a pure exploration warmup period provided by \cref{alg:reward_free} (which we describe in more detail in \cref{sec:rfw}), and the usage of \emph{restricted} value functions.
The restricted value functions, in contrast to truncated ones, take zero value outside the confidence state set.

The core property required from the warmup period is that the data it collects is sufficient to ensure a small error when using it in the least squares regression step of \cref{alg:oppo_linmdp}. 
The degree to which the error should be small is determined by the multiplicative factor in the confidence bound for a single regression step (determined by the bonus parameter $\beta$ along with other problem parameters), and the number of times we perform this step ($H$; the length of the horizon). The analysis leads to the following definition for the ``known'' states set of step $h$:
\begin{align}
    \;\cZ_h \eqq \cbr{s\in \cS \mid \forall a, \norm{\phi(s, a)}_{\Lambda_{0,h}^{-1}} \leq 1/(2 \beta H)},
    \label{eq:def_known_states}
\end{align}
where $\Lambda_{0, h}$ denotes the warmup covariate matrix returned by \cref{alg:reward_free} for step $h$. The set $\cZ_h$ contains the states for which we collected enough data, so that the least squares regression error when estimating their value
can be well controlled without employing truncation.

On episode $k$, the standard optimistic estimates value function estimates are denoted $\widetilde Q_h^k, \widetilde V_h^k$, while their restricted counterparts are defined by:
\begin{align*}
    \widetilde Q^{k;\circ}_h(s, a) 
    &=  \I\cbr{s \in \cZ_h}\widetilde Q^k_h(s, a),
    \\
\widetilde V_h^{k;\circ}(s) 
					&= \abr{\widetilde Q^{k;\circ}_h(s, \cdot), \pi^k_h(\cdot| s)}.
\end{align*}
During the backward dynamic programming step, the estimate of the non-restricted action-value function $\widetilde Q_{h-1}^k$ then makes use of the least squares solution w.r.t.~ the restricted $\widetilde V_{h}^{k;\circ}$, which has a well bounded $\norm{\cdot}_\infty$. Further, the warmup ensures the known state set $\cZ_h$ is large enough so that we do not lose much by this restriction; concretely, that no policy has total occupancy larger than $O(\epsilon_{\rm cov})$ outside the known states set.
% Owed to the use of non-truncated action-value functions in the policy evaluation step, we are able to control the value class capacity without batching approaches that lead to suboptimal rates.

The other important ingredient of \cref{alg:oppo_linmdp} is the epoch schedule in the updates of bonus functions $\hat b_h^k$, determined by the determinant of the covariate matrices $\Lambda_{k, h}$. This ensures we update the bonus functions at most $O(\log K)$ times, which, when combined with the truncation-less least squares routine, allows keeping the number of variables in the policy parameterization $O(d^2\log K)$.
We conclude this section with our main theorem, providing the regret guarantees of \cref{alg:oppo_linmdp}.

\begin{theorem} 
\label{thm:oppo_regret}
	Let $\delta > 0$, assume
 $K\geq  H^5 d^4 \log^8 (dHK/\delta)$, $H \geq 3$, $\log A \leq K$,
 and consider setting
 $\beta = 2 c_\beta d^{3/2} H \log(d H K/\delta)$ where $c_\beta$ is specified by \cref{lem:good_event},
	$\epsilon_{\rm cov} = {H^{3/2} d^2 \log^4(dHK/\delta)/\sqrt K}$
	and $\eta = \sqrt{ \log A}/\br[s]{H \sqrt K}$.
	Suppose we run \cref{alg:oppo_linmdp} with these parameters for either the adversarial case with full-feedback (\cref{assume:linmdp_adv}), or the stochastic case with bandit-feedback (\cref{assume:linmdp_stochastic}).
 Then we obtain the following bound w.p. $1-4\delta$:
	\begin{align*}
		\sum_{k=1}^K V^{k, \pi^k} - V^{k, \pi^\star}
		&= O \br{  
			 d^2 H^{7/2}\log^{4}\frac{d H K}{\delta} \sqrt {K \log A}
		},
	\end{align*}
	where big-$O$ hides only constant factors independent of problem parameters.
\end{theorem}

\begin{algorithm}[ht!]
    \caption{Optimistic PO for Linear MDPs} 
    \label{alg:oppo_linmdp}
	\begin{algorithmic}
	    \STATE \textbf{input:} $
	    	\br { \eta, \delta, \beta, \epsilon_{\rm cov}}. 
	    $
		\STATE $\cbr{\br[b]{\cD_h^{0}, \Lambda_{0, h}}}_{h\in [H]} \gets $  \cref{alg:reward_free} ($\delta, \beta, \epsilon_{\rm cov}$)
            \STATE Let $K_0 -1$ be the number of rounds \cref{alg:reward_free} played
		\STATE Init $\forall s:\; \pi_h^1(\cdot|s) = {\rm Unif}(\cA)$, $\forall h\in[H]:\;\widehat \Lambda_{K_0, h} = 0$.
  
\FOR{$k=K_0, \ldots, K$}
%	    	\STATE
    \STATE Rollout $\pi^k$ to generate $\cbr{(s_h^k, a_h^k, \l_h^k)}_{h=1}^H$.
    
    % \STATE {\color{gray}\# Policy evaluation:}
    \STATE $\widetilde V^k_{H+1}(\cdot) \equiv 0	$.
    \FOR{$h=H, \ldots, 1$}
            \STATE $\cD_h^k \gets \cD_h^0 \cup \cbr{(s_h^i, a_h^i, s_{h+1}^i)}_{i=K_0}^{k-1}$
        \STATE$\Lambda_{k,h} \gets 
                I + \sum_{i\in \cD_h^k} \phi(s_h^i, a_h^i) \phi(s_h^i, a_h^i)\T$

        \IF{$\det \Lambda_{k,h} \geq 2\det \widehat \Lambda_{k,h}$}
        \STATE $\widehat \Lambda_{k,h} \gets \Lambda_{k,h} $
        \STATE $\hat b_h^k(s, a) = \beta 
        \sqrt{\phi(s, a)\T \widehat \Lambda_{k,h}^{-1} \phi(s, a)}$
        \ENDIF
        
        \STATE $\widehat v_h^k \gets \Lambda_{k, h}^{-1} 
            \sum_{i\in \cD_{h}^k}
                \phi(s_h^i, a_h^i)\widetilde V^{k;\circ}_{h+1} (s_{h+1}^i)$
        \STATE $\widehat \P_h^k \widetilde V_{h+1}^{k;\circ}(s, a)=\phi(s, a)\T \widehat v_h^k$
        \STATE $\widehat g_{k, h} \gets \begin{cases}
            \textit{Adv.:} 
            & g_{k, h}
            \\
            \textit{Stoch.:} 
            &\Lambda_{k, h}^{-1} 
            \sum_{i\in \cD_{h}^k}
                \phi(s_h^i, a_h^i) \l_h^i(s_h^i, a_h^i)
        \end{cases}$
        \STATE     $\widehat \l_h^k(s, a) = \phi(s, a)\T \widehat g_{k, h}$
        \STATE Set
        \begin{align*}
            \begin{cases}
                \widetilde Q^k_h(s, a) 
            &=  \widehat \l_h^k(s, a) + \widehat \P_h^k \widetilde V_{h+1}^{k;\circ}(s, a) - \hat b_h^k(s, a)
            \\
            \widetilde Q^{k;\circ}_h(s, a) 
            &=  \I\cbr{s \in \cZ_h}\widetilde Q^k_h(s, a)
            \\
            \widetilde V_h^k(s) 
            &= \abr{\widetilde Q^k_h(s, \cdot), \pi^k_h(\cdot| s)}
            \\
            \widetilde V_h^{k;\circ}(s) 
            &= \abr{\widetilde Q^{k;\circ}_h(s, \cdot), \pi^k_h(\cdot| s)}
            \end{cases}
        \end{align*}
    \ENDFOR
    \STATE
    {\color{gray}\# Policy improvement:}
    \begin{align*}
        \pi_h^{k+1}(a|s)
        &\propto \pi_h^{k}(a|s) e^{-\eta \widetilde Q_h^k(s, a)} 
    \end{align*}
    
\ENDFOR
	\end{algorithmic}
\end{algorithm}

\subsection{Reward-free warmup}
\label{sec:rfw}

In this section we present \cref{alg:reward_free}, which we employ for a pure exploration warmup period. The algorithm invokes the \CoverTraj~algorithm \cite{wagenmaker2022reward} for each step of the horizon, and thus follows the same high level design of reward free exploration outlined in Algorithm 1 of \citet{wagenmaker2022reward}.

The basic guarantee provided by the warmup period is given by the next lemma.

\begin{lemma}
\label{lem:goodevent_rfw}
	Assume we execute \cref{alg:reward_free} with the setting of $\beta = \widetilde O(d^{3/2} H)$ and $\epsilon_{\rm cov}\geq 1/K$.
	Then it will terminate after $O\br{\frac{d^4 H^5}{\epsilon_{\rm cov}} \log^{7}\frac{d H K}{\delta} }$ episodes, and with probability $\geq 1-\delta$, outputs 
	$\Lambda_{0, 1}, \ldots, \Lambda_{0, H}$ such that:
	\begin{align*}
		\forall h, \forall \pi, 
  \Pr_{s_h \sim \mu_h^\pi}\br{s_h \notin \cZ_h} \leq \epsilon_{\rm cov}.
	\end{align*}
    % where $\cZ_h$ is defined in \cref{eq:def_known_states}.
\end{lemma}
The proof of \cref{lem:goodevent_rfw} is provided in \cref{sec:goodevent_rfw}, and mostly follows from the basic guarantees of the \CoverTraj~algorithm.

\begin{algorithm}[ht!]
    \caption{Reward Free Warmup} 
    \label{alg:reward_free}
	\begin{algorithmic}
		\STATE \textbf{input:} $\delta,  \beta, \epsilon_{\rm cov}$
		\STATE Set $m=\lceil \log\frac{1}{\epsilon_{\rm cov}} \rceil$
		\STATE Set $\forall i\in[m], \; \gamma_i = 1/(2 \beta H)$
            \STATE $\cZ_{H+1} \eqq \cbr{s_{H+1}}$
		\FOR {$h=H, \ldots, 1$}			
			\STATE $\cbr{\br{\cX_{h, i}, \widetilde \cD_{h,i}, \widetilde \Lambda_{h,i}}}_{i=1}^m \gets \CoverTraj (h, \delta/H, m, \cbr{\gamma_i})$
			
			\STATE $\cD_h^0 \gets \bigcup_i \widetilde \cD_{h, i}$
            \STATE
			$\Lambda_{0,h}  \gets 
			I + \sum_{t\in \cD_h^0}\phi(s_h^t, a_h^t)\phi(s_h^t, a_h^t)\T $
		\ENDFOR
		\STATE \textbf{return}  $\cbr{\br{\cD_h^0,
                    	\Lambda_{0,h}
                }}_{h\in [H]}$
	\end{algorithmic}
\end{algorithm}

\section{Proof of Main Theorem}
In this section, we outline the technical arguments leading up to the proof of \cref{thm:oppo_regret}.
At the core of most of the analyses of linear MDP algorithms that involve a value estimation step,
is a uniform convergence argument that ensures the regression errors concentrate uniformly over the class of value functions explored by the algorithm. The need for uniform convergence stems from the fact that we estimate the value function using past rollouts and a previously estimated value function, which in itself depends on past rollouts through the current agent policy. 
The lemma below is used to establish this part of the argument, and is stated in a generic manner --- this is essentially the same argument used in \citet{jin2020provably}.
\begin{lemma}	\label{lem:value_concentration_base}
	Let $\cV \subseteq \cS \to \R$ be a class of functions where $\forall f\in \cV, \norm{f}_\infty \leq C$, fix $h\in [H]$, and consider a transitions dataset $\cD_h =\cbr{(s_h^i, a_h^i, s_{h+1}^i)}_{i\in[k]}$ collected by agent rollouts in the environment.
    Let $\widehat \P_h f (s, a) = \phi(s, a)\T \hat v_h^f$ be the approximation of $\P_h f (s, a) = \E_{s' \sim \P_h(\cdot|s, a)} f(s')$
	given by the least squares estimate $\hat v_h^f =\Lambda_h^{-1}\sum_{i=1}^k \phi(s_h^i, a_h^i) f(s_{h+1}^i)$,
    $\Lambda_h \eqq I + \sum_{i=1}^k \phi(s_h^i, a_h^i) \phi(s_h^i, a_h^i)\T$.
    Then, w.p.~$1-\delta$ over the generation of $\cD_h$, we have $\forall s,a\in \cS\times \cA, \forall f\in \cV$;
	\begin{align*}
		&\av{\br[b]{\P_h - \widehat \P_{h}} f(s, a)}
		% \\ %icmledit
  %           &
            \leq 
			\br{4 C \sqrt{2d \log k 
			+ \log\frac{\cN_{1/k}(\cV)}{\delta}}} \norm{\phi(s, a)}_{\Lambda_h^{-1}}
		,
	\end{align*}
 where $\cN_\nu(\cV)$ denotes the $\norm{\cdot}_\infty$ covering number of $\cV$.
\end{lemma}
Compared to other works, our analysis differs most significantly in how the above lemma is applied in order to control the estimation errors of the algorithm; the important parameter being the bound $C$ on the magnitude of the target function.
Usually, the regression step for $\widetilde V_h^k$ uses a truncated version of $\widetilde V_{h+1}^k$ from the previous horizon step, which simplifies the analysis and unfortunately leads to a too large value function class.
The core of our argument lies in the next lemma; specifically, in the proof of \ref{def:goodevent_vbu}.
\begin{lemma}[The good event]
    \label{lem:good_event}
    There exists a universal constant $c_\beta$, such that for any $\delta > 0$, executing \cref{alg:oppo_linmdp} with $\beta \geq c_\beta d^{3/2} H \log(d H K/\delta)$, we have that the following hold w.p.~$>1-4\delta$.
    
    $\forall \pi, \forall h$:
    \begin{align}
	\Pr_{s_h \sim \mu_h^\pi}(s_h \notin \cZ_h)
        \leq \epsilon_{\rm cov},
        \tag{$\cE^{\rm rfw}$}
        \label{def:goodevent_rfw}
    \end{align}
    $\forall k \geq K_0, h, s, a$:
    \begin{align}
        \av{\widetilde Q_h^{k;\circ}(s, a)} 
        &\leq 2 H,
        \tag{$\cE^{\rm qbd}$}
        % \label{def:goodevent_Qbounded}
        \label{def:goodevent_qbd}
        \\
	|(\P_h - \widehat \P_{h}^{k}) \widetilde V^{k;\circ}_{h+1}(s, a)|
			&\leq
			\frac\beta2
			\norm{\phi(s, a)}_{\Lambda_{k,h}^{-1}},
   \tag{$\cE^{\rm vbu}$}
    \label{def:goodevent_vbu}
        \\
            |\widehat \l_{h}^k(s, a) - \l_h^k(s, a)|
			&\leq
			\frac\beta2
			\norm{\phi(s, a)}_{\Lambda_{k,h}^{-1}},
        \tag{$\cE^{\rm sle}$}
        \label{def:goodevent_sle}
    \end{align}
    and $\forall h\in [H]$:
    \begin{align}
    \sum_{k=K_0}^K & \E_{s_h, a_h \sim \mu_h^k}\sbr{
				\norm[b]{\phi(s_h, a_h)}_{\Lambda_{k,h}^{-1}}
			}
	% \nonumber \\ %icmledit
 %            &
            \leq 2 \sum_{k=K_0}^K
				\norm[b]{\phi(s_h^k, a_h^k)}_{\Lambda_{k,h}^{-1}}
				+
				4 \log \frac{4 K H}{\delta}.
    \label{def:goodevent_bon}
    \tag{$\cE^{\rm bon}$}
    \end{align}
\end{lemma}
The proof of \cref{lem:good_event} is provided fully in \cref{sec:goodevent}; below we provide an overview.
The success of \ref{def:goodevent_rfw} is given by \cref{lem:goodevent_rfw}, 
while the proofs for \ref{def:goodevent_sle} and \ref{def:goodevent_bon} 
follow from standard arguments; we provide their proofs in \cref{lem:goodevent_lse,lem:goodevent_bonus}, respectively.
% The proof sketch given below outlines the principal part of our technical argument.
\begin{proof}[sketch (\ref{def:goodevent_rfw} $\cup$ \ref{def:goodevent_sle} 
    $\implies$ \ref{def:goodevent_qbd} $\cup$ \ref{def:goodevent_vbu})] 
    Establishing the bound in \ref{def:goodevent_vbu} involves showing (i) \ref{def:goodevent_qbd} holds for $\widetilde Q_{h+1}^{k;\circ}$, and (ii) that the policy $\pi_{h+1}^k$ belongs to a ``small'' policy class.
    Given (i) and (ii), it immediately follows that $\widetilde V^{k;\circ}_{h+1}$ belongs to a small and bounded value function class, 
    which leads to \ref{def:goodevent_vbu}
    through an application of \cref{lem:value_concentration_base}.

    We proceed by an inductive argument as follows. Let $k, h$, and assume we have already proved (i), (ii) and \ref{def:goodevent_vbu} for $(k',h'), k'<k$ and $(k,h'), h'>h$.
	 Our Q estimate on step $h$ decomposes as:
          \begin{align*}
			\av{\widetilde Q_{h}^{k;\circ}(s, a)}
			&= 
			\av{\widehat \l_h^k(s, a)
			+ \widehat \P_h^k \widetilde V_{h+1}^{k;\circ} (s, a)
			- \hat b_h^k(s, a)}.
		\end{align*}
		Now we may apply \ref{def:goodevent_sle} for the loss term, and the inductive hypothesis combined with \ref{def:goodevent_rfw} for the regression term; which gives us that it is close to the true and well bounded value --- this gives (i).
		For (ii), we employ the inductive hypothesis for $k'<k$ to show that the policy $\pi_h^k$ has compact parametric form.
  As mentioned above, (i) + (ii) now lead to \ref{def:goodevent_vbu},
	which completes the inductive step and thus the proof.
\end{proof}

Throughout the remainder of the analysis, we let $\cK_h$ denote the event that a random state $s_h$ is ``known'':
    \begin{align}
        \cK_h \eqq \cbr{s_h \in \cZ_h}.
    \end{align}
\begin{lemma}[Regret decomposition]
\label{lem:regret_decomposition}
	Upon execution of \cref{alg:oppo_linmdp}, conditioned on the good event \cref{lem:good_event}, it holds that:
	\begin{align*}
		&\sum_{k=K_0}^K V^{k, \pi^k} - V^{k, \pi^\star}
            \leq 4 \epsilon_{\rm cov} H^2 K
		\\
            &+
		\underbrace{\sum_{h=1}^H \sum_{k=K_0}^K  
			\E_{s_h, a_h \sim \mu_h^k}\sbr{
				- \Delta_h^k(s_h, a_h)
				+ \hat b_h^k(s_h, a_h)
				\mid \cK_h
			}
		}_\text{Bias}
		\\
		& + 
		\underbrace{\sum_{h=1}^H \sum_{k=K_0}^K 
		\E_{s_h \sim \mu_h^\star}\sbr{
			\abr{\widetilde Q^{k}_h(s_h, \cdot), 
			\pi_h^k(\cdot|s_h) - \pi_h^\star(\cdot|s_h)
			}
			\mid \cK_h
		}}_\text{OMD}
		\\
		&+
		\underbrace{\sum_{h=1}^H \sum_{k=K_0}^K  
			\E_{s_h, a_h \sim \mu_h^\star}\sbr{
				\Delta_h^k(s_h, a_h)
				- \hat b_h^k(s_h, a_h)
				\mid \cK_h
			}
		}_\text{Optimism}
		,
	\end{align*}
        where 
        \begin{align*}
        \Delta_h^k(s_h, a_h) 
            &\eqq  
            \widehat \l_h^k(s_h, a_h) - \l_h^k(s_h, a_h)
            % \\ %icmledit
            % &\quad 
            + \br{\widehat \P_h^k - \P_h} \widetilde V_{h+1}^{k;\circ} (s_h, a_h).    
        \end{align*}
\end{lemma}

\begin{proof}[sketch]
	For any $k$, we have by the extended value difference lemma (\cref{lem:extended_value_diff});
	\begin{align*}
		&V_1^{k, \pi^k} - V_1^{k, \pi^\star}
		\\
		&=
		\sum_{h=1}^H \E_{\mu_h^k}\sbr{
			\l_h^k(s_h, a_h) + \P_h \widetilde V_{h+1}^{k;\circ}(s_h, a_h)
			- \widetilde Q^{k;\circ}_h(s_h, a_h)
		}
		\\
		& + 
		\sum_{h=1}^H \E_{\mu_h^\star}\sbr{
			\abr{\widetilde Q^{k;\circ}_h(s_h, \cdot), 
			\pi_h^k(\cdot|s_h) - \pi_h^\star(\cdot|s_h)
			}
		}
		\\
		& +
		\sum_{h=1}^H  \E_{\mu_h^\star}\sbr{
			\widetilde Q^{k;\circ}_h(s_h, a_h) 
			- \l^k_h(s_h, a_h)
				- \P_h \widetilde V_{h+1}^{k;\circ}(s_h, a_h)
		}.
	\end{align*}
	Now, note that for any $s\in \cZ_h, a\in \cA$;
	\begin{align*}
		\widetilde Q_h^{k;\circ}(s, a)
		=
		\phi(s, a)\T \widehat g_{k, h}
		+ \widehat \P_h^k \widetilde V_h^{k;\circ}(s, a)
		-\hat b_h^k(s, a),
	\end{align*}
	thus
	\begin{align*}
		\l_h^k(s, a) + \P_h \widetilde V_{h+1}^{k;\circ}&(s, a)
			 - \widetilde Q^{k;\circ}_h(s, a)
		% \\ %icmledit
  %           &
            = 
		-\Delta_h^k(s, a) + \hat b_h^k(s, a).
	\end{align*}
	In addition, by the good event, specifically \ref{def:goodevent_qbd}, and the assumption that the instantaneous loss is $\in[-1, 1]$, we have 
	for any $s\notin \cZ_h, a\in \cA$:
	\begin{align*}
		\av{
			\l_h^k(s, a) + \P_h \widetilde V_{h+1}^{k;\circ}(s, a)
			- \widetilde Q^{k;\circ}_h(s, a)
		}
		\leq 2H.
	\end{align*}
	Thus by the law of total expectation,
	\begin{align*}
		\E_{\mu_h^k}&\sbr{
			\l_h^k(s_h, a_h) + \P_h \widetilde V_{h+1}^{k;\circ}(s_h, a_h)
			- \widetilde Q^{k;\circ}_h(s_h, a_h)
		}
		\\
		&\leq
		\E_{\mu_h^k}\sbr{
		-\Delta_h^k(s_h, a_h)
			+ \hat b_h^k(s, a) 
			\mid \cK_h
		}
		+ 2 \epsilon_{\rm cov} H,
	\end{align*}
	where the inequality follows since the good event \cref{def:goodevent_rfw} implies $\mu_h^k(\cS \setminus \cZ_h) \leq \epsilon_{\rm cov}$. For similar reasons, we also have,
	\begin{align*}
		\E_{\mu_h^\star}&\sbr{
			\widetilde Q^{k;\circ}_h(s_h, a_h) 
			- \l^k_h(s_h, a_h)
				- \P_h \widetilde V_{h+1}^{k;\circ}(s_h, a_h)
		}
		\\
		&\leq
		\E_{\mu_h^\star}\sbr{
		      \Delta_h^k(s_h, a_h)
			-\hat b_h^k(s, a)
			\mid \cK_h
		}
		+ 2 \epsilon_{\rm cov} H.
	\end{align*}
	Finally, again by the law of total expectation and definition of the restricted Q-function;
	\begin{align*}
		&\E_{s_h \sim \mu_h^\star}\sbr{
			\abr{\widetilde Q^{k;\circ}_h(s_h, \cdot), 
			\pi_h^k(\cdot|s_h) - \pi_h^\star(\cdot|s_h)
			}
		}
		\\
            &\leq 
		\E_{s_h \sim \mu_h^\star}\sbr{
			\abr{\widetilde Q^{k}_h(s_h, \cdot), 
			\pi_h^k(\cdot|s_h) - \pi_h^\star(\cdot|s_h)
			}
			\mid \cK_h
		}.
	\end{align*}
	Combining the last three displays with the first equation, and summing over $k=K_0, \ldots, K$, $h\in [H]$, the proof is complete.
\end{proof}

% The proof of the above lemma \us{given in cref} follows from the extended value difference \cref{lem:extended_value_diff} combined with the reward free success event \cref{def:goodevent_rfw} and  boundedness of the restricted value functions both of which are given by \cref{lem:good_event}.
We are now ready for the proof of the main theorem.
\begin{proof}[of \cref{thm:oppo_regret}]
	Given our choice of parameters, by \cref{lem:goodevent_rfw}, we have that the number of warmup episodes satisfies
	\begin{align}
		\label{eq:K0_warmup_episodes}
		K_0 = O\br{\frac{d^4 H^5}{\epsilon_{\rm cov}} \log^{7}\frac{d H K}{\delta} }.
	\end{align}
 For the remainder of the proof, we assume the good event defined in \cref{lem:good_event} holds, which indeed occurs w.p.~$1-4\delta$ by that lemma.
Proceeding, we will bound the regret for the remaining rounds using the decomposition given by \cref{lem:regret_decomposition}.

\paragraph{Bias term.}
By \cref{def:goodevent_vbu}, we have for all $s,a$:
	\begin{align*}
		\br[b]{\P_h - \widehat \P_h^k}\widetilde V_{h+1}^{k;\circ}(s, a)
		&+ \frac12\hat b_h^k(s, a)
		% \\ %icmledit
  %           &
            \leq 
		\frac\beta2 \norm{\phi(s, a)}_{\Lambda_{k, h}^{-1}} 
		+ \frac12\hat b_h^k(s, a).
	\end{align*}
 In addition, in the stochastic case,
 owed to \cref{def:goodevent_sle}, for all $s, a$:
 \begin{align*}
	\l_h(s, a) - \widehat \l_h^k(s, a)
        &=
        \phi(s_h, a_h)\T\br{g_{k, h} - \widehat g_{k, h}}
        + \frac12\hat b_h^k(s, a)
        \\
        &\leq 
        \frac\beta2 \norm{\phi(s, a)}_{\Lambda_{k, h}^{-1}} 
        + \frac12\hat b_h^k(s, a).
\end{align*}
In the adversarial case, the above bound holds trivially since
$\l_h(s, a) = \widehat \l_h^k(s, a)$.
	By a simple algebraic argument given in  \cref{lem:bonus_doubling}, we additionally have
	\begin{aligni*}
		\hat b_h^k(s, a) = \beta\norm{\phi(s, a)}_{\widehat \Lambda_{k, h}^{-1}}
		\leq 
		2\beta\norm{\phi(s, a)}_{ \Lambda_{k, h}^{-1}}
		,
	\end{aligni*}
	thus the sum of the last two displays is bounded by $3\beta\norm{\phi(s, a)}_{\Lambda_{k, h}^{-1}}$, therefore
	\begin{align*}
		\text{Bias} 
                &\leq 3\beta \sum_{h=1}^H \sum_{k=K_0}^K  
			\E_{\mu_h^k}\sbr{ 
			\norm[b]{\phi(s_h^k, a_h^k)}_{\Lambda_{k,h}^{-1}} }
		\\
            &\leq 
		6 \beta \sum_{h=1}^H \sum_{k=K_0}^K
			\norm[b]{\phi(s_h^k, a_h^k)}_{\Lambda_{k,h}^{-1}}
			+
			12 \beta H\log \frac{4 K H}{\delta}
		,
	\end{align*}
	where the second inequality follows from \cref{def:goodevent_bon}.
	Further, by \cref{lem:eliptical_potential}, for any $h\in[H]$,
	\begin{align*}
		\sum_{k=K_0}^K 
			\norm[b]{\phi (s_h^k, a_h^k) }_{\Lambda_{k,h}^{-1}}
		&\leq \sqrt {K \sum_{k=K_0}^K 
			\norm[b]{\phi (s_h^k, a_h^k) }_{\Lambda_{k,h}^{-1}}^2 }
		% \\ %icmledit
  %           &
            \leq 2 \sqrt {K d \log K},
	\end{align*}
	hence,
	\begin{align*}
		\text{Bias}
		\leq 12 \beta H\br{ \sqrt{K d \log K} + \log \frac{4 K H}{\delta} }.
	\end{align*}

	\paragraph{OMD Term.} 
	By \cref{def:goodevent_qbd} we have that 
	for all $k\geq K_0, h, s\in \cZ_h, a\in \cA$;
	\begin{aligni*}
		\av[b]{\widetilde Q^{k}_h(s, a)}
		=
		\av[b]{\widetilde Q^{k;\circ}_h(s, a)} \leq 2H
	\end{aligni*}
	.
	Thus, applying the OMD regret bound \cref{lem:omd} for any $s\in \cZ_h, h\in [H]$ we have;
	\begin{align*}
		\sum_{k=K_0}^K &\abr{
			\widetilde Q^{k}_h(s, \cdot), 
			\pi_h^k(\cdot|s) - \pi_h^\star(\cdot|s)
		}
		\\
            &\leq \frac{\log A}{\eta}
			+ \eta \sum_{k=K_0}^K \sum_{a\in \cA} \pi_h^k(a|s)
				\widetilde Q_h^{k}(s, a)^2
		\\
            &\leq \frac{\log A}{\eta}
			+ 4\eta H^2 K.
	\end{align*}
	Therefore, we may bound the OMD term as follows:
	\begin{align*}
            &\sum_{h=1}^H \sum_{k=K_0}^K \E_{s_h \sim \mu_h^\star}\sbr{
				\abr{\widetilde Q^k_h(s_h, \cdot), 
				\pi_h^k(\cdot|s_h) - \pi_h^\star(\cdot|s_h)
				} \mid \cK_h}
			\\
				&=  \sum_{h=1}^H \E_{s_h \sim \mu_h^\star}
				\sbr{
				\sum_{k=K_0}^K
				\abr{\widetilde Q^k_h(s_h, \cdot), 
				\pi_h^k(\cdot|s_h) - \pi_h^\star(\cdot|s_h)
				}
				\mid \cK_h}
			\\
			&\leq \sum_{h=1}^H \E_{s_h \sim \mu_h^\star}\sbr{
			\frac{\log A}{\eta}
			+ 4\eta H^2 K
			\mid \cK_h}
			\\
			&= \frac{H \log A}{\eta} + 4 \eta H^3 K.
	\end{align*}

	\paragraph{Optimism term.}
	By \cref{def:goodevent_vbu}, for $s,a\in \cZ_h\times\cA$:
	\begin{align*}
		\br[b]{\widehat \P_h^k - \P_h}
            &\widetilde V_{h+1}^{k;\circ}(s, a)
		- \frac12 \hat b_h^k(s, a)
		% \\ %icmledit
  %           &
            \leq 
		\frac\beta2 \norm{\phi(s, a)}_{\Lambda_{k, h}^{-1}} 
		- \frac\beta2 \norm{\phi(s, a)}_{\widehat \Lambda_{k, h}^{-1}} 
		\leq 0,
	\end{align*}
	since $\widehat \Lambda_{k, h}^{-1} \succeq \Lambda_{k, h}^{-1}$ by  construction.
  Similarly,
 owed to \cref{def:goodevent_sle}:
 \begin{align*}
%	\widehat \l_h^k(s, a) - \l_h(s, a) 
%	- \frac12 \hat b_h^k(s, a)
%        =
        \phi(s, a)&\T\br{\widehat g_{k, h} - g_{k, h}}
        - \frac12 \hat b_h^k(s, a)
        % \\ %icmledit
        % &
        \leq 
		\frac\beta2 \norm{\phi(s, a)}_{\Lambda_{k, h}^{-1}} 
		- \frac\beta2 \norm{\phi(s, a)}_{\widehat \Lambda_{k, h}^{-1}} 
		\leq 0.
\end{align*}
	Thus, we immediately obtain the optimism term is non positive.
	\paragraph{Concluding the proof.}
	Combining the bound on the number of warmup episodes \cref{eq:K0_warmup_episodes},
	with \cref{lem:regret_decomposition} and the bounds on all three terms, we have:
	\begin{align*}
		\sum_{k=1}^K V^{\pi^k} - V^\star
		&\lesssim 
		\frac{d^4 H^5}{\epsilon_{\rm cov}} \log^{7}\frac{d H K}{\delta} 
		+ \epsilon_{\rm cov} H^2 K
		+ \frac{H \log A}{\eta} 
            \\
            &+ \eta H^3 K
		+  \beta H\br{ \sqrt{K d \log K} + \log \frac{K H}{\delta} },
	\end{align*}
	where $\lesssim$ hides only constant factors.
	Finally, setting $\epsilon_{\rm cov} = \frac{H^{3/2} d^2 \log^4(dHK/\delta)}{\sqrt K}, \beta = 2 c_\beta d^{3/2} H \log(d H K/\delta)$
	and $\eta = \frac{\sqrt{ \log A}}{H \sqrt K}$, we obtain:
	\begin{align*}
		\sum_{k=1}^K V^{\pi^k} - V^\star
		% &\lesssim % ANOTHER STEP HERE
		% d^2 H^{7/2} \sqrt K\log^{4}\frac{H K}{\delta} 
		% + H^2 \sqrt{K \log A} 
		% +  d H^2 \log(dHK/\delta) \br{ \sqrt{K d \log K} + \log \frac{K H}{\delta} }
		% \\
		&\lesssim
		d^2 H^{7/2} \log^{4}\frac{d H K}{\delta} \sqrt {K \log A}
		,
	\end{align*}
	which completes the proof.
\end{proof}

% \newpage

% \section*{Impact Statement}
% This paper presents work whose goal is to advance the field of theoretical Machine Learning. There are many potential societal consequences of our work, none which we feel must be specifically highlighted here.

\section*{Acknowledgements}
The authors thank Asaf Cassel for spotting a bug in an earlier version of this manuscript.
This project has received funding from 
the European Research Council (ERC) under the European Union’s Horizon 2020 research and innovation program (grant agreements No. 101078075; 882396). Views and opinions expressed are however those of the author(s) only and do not necessarily reflect those of the European Union or the European Research Council. Neither the European Union nor the granting authority can be held responsible for them.
This project has also received funding from 
the Israel Science Foundation (ISF, grant numbers 2549/19;  2250/22), 
the Yandex Initiative for Machine Learning at Tel Aviv University,
the Tel Aviv University Center for AI and Data Science (TAD), 
the Len Blavatnik and the Blavatnik Family foundation, 
and from the Adelis Foundation.

% \section*{Impact Statement}
% %icmledit
% This paper presents work whose goal is to advance the field of 
% Machine Learning. There are many potential societal consequences 
% of our work, none which we feel must be specifically highlighted here.

\bibliography{main}
%%%%%%%%%%%%%%%%%%%%%%%%%%%%%%%%%%%%%%%%%%%%%%%%%%%%%%%%%%%%%%%%%%
% APPENDIX
%%%%%%%%%%%%%%%%%%%%%%%%%%%%%%%%%%%%%%%%%%%%%%%%%%%%%%%%%%%%%%%%%%
\newpage
\appendix
\onecolumn

\section{Deferred Proofs}
In this section, we provide details of the analysis that weren't fully included in the main text. 
The central component is the proof of \cref{lem:good_event}, which is given in \cref{sec:goodevent}.
In \cref{sec:value_policy_classes} we define the value and policy classes which will be shown in \cref{sec:goodevent} to contain (w.h.p.) the values and policies explored by the algorithm. 
\cref{sec:empirical_value_covering} includes the technical details for the covering number bounds of the value classes defined in \cref{sec:value_policy_classes}.

\paragraph{Additional notation.}
We will make use of the following filtration;
\begin{align}
    \label{eq:def:filtration}
    \cF^{k}_h \eqq \sigma\br{
        (s_{h'}^1, a_{h'}^1, l_{h'}^1)_{h'=1}^H, 
        \ldots, 
        (s_{h'}^{k-1}, a_{h'}^{k-1}, l_{h'}^{k-1})_{h'=1}^H
        ,
        (s_{h'}^{k}, a_{h'}^{k}, l_{h'}^{k})_{h'=1}^h
    },\;
    \cF^k \eqq \cF^k_H,
\end{align}
where $(s_h^i, a_h^i, l_h^i)$ are the (state, action, loss) random variables generated during policy rollouts.
In addition, for any function class $E\subseteq {\cX \to \R}$, where $\cX$ is an arbitrary set, we let $\cN_\nu(E)$ denote the $\norm{\cdot}_\infty$ covering number of $E$; that is, the cardinality of the smallest set $\widetilde E \subset E$ such that for all $f\in E$, there exists $\tilde f\in \widetilde E$ such that $\max_{x\in \cX} \av[b]{f(x) - \tilde f(x)}\leq \nu$.

\subsection[Proof of value-concentration]{Proof of \cref{lem:value_concentration_base}}
\label{sec:proof:value_concentration_base}
\begin{proof}[of \cref{lem:value_concentration_base}]
	Denote
\begin{align*}
	v_h \eqq \int \psi_h(s') f(s'){\rm d}s'
	; \quad 
	\hat v_h \eqq 
	\Lambda_{h}^{-1} \sum_{i = 1}^{k} \phi(s_h^i, a_h^i) f(s_{h+1}^i).
\end{align*}
Then, we have
\begin{align}
	\hat v_h - v_h
	&= \Lambda_{h}^{-1} \br{ 
		\sum_{i = 1}^{k} \phi(s_h^i, a_h^i) f(s_{h+1}^i)
		- \br[Bg]{ I + \sum_{i = 1}^{k} \phi(s_h^i, a_h^i)\phi(s_h^i, a_h^i)\T }
			v_h
	}
	\nonumber
	\\
	&= \Lambda_{h}^{-1} \sum_{i = 1}^{k} \phi(s_h^i, a_h^i) \br{ 
		f(s_{h+1}^i) - \phi(s_h^i, a_h^i)\T v_h
	}
	-  \Lambda_{h}^{-1} v_h
	\nonumber
	\\
	&= \Lambda_{h}^{-1} \sum_{i = 1}^{k} \phi(s_h^i, a_h^i) \br{ 
		f(s_{h+1}^i) - \E_{s'}\sbr{f(s') \mid s_h^i, a_h^i}
	}
	- \Lambda_{h}^{-1} v_h
	\label{eq:valcon_err}
\end{align}
Now, note that
\begin{align*}
	\norm{ \Lambda_{h}^{-1} v_h}_{\Lambda_{h}}^2
	= \norm{ v_h}_{\Lambda_{h}^{-1}}^2
	\leq \norm{ v_h}^2
	\leq d C^2.
\end{align*}
In addition, for the first term in \cref{eq:valcon_err}, we consider the filtration  defined in \cref{eq:def:filtration}, and
note that $\phi(s_h^i, a_h^i)$ is $\cF^i_h$-measurable while $s_{h+1}^i$ is $\cF_{h+1}^i$-measurable. Hence
we may apply \cref{lem:ols_unif_concentration} to obtain that for any $\epsilon, p > 0$, with probability $\geq 1-p$ we have
\begin{align*}
&\norm{\Lambda_{h}^{-1} \sum_{i = 1}^{k} \phi(s_h^i, a_h^i) \br{ 
		f(s_{h+1}^i) - \E_{s'}\sbr{f(s') \mid s_h^i, a_h^i}
	}}_{\Lambda_h}^2 
	\\
	&\qquad\leq 4 C^2 
	\br{\frac{d}2 \log\br{k + 1} 
	+ \log\frac{\cN_\epsilon(\cV)}{p}}
	+ 8 k^2 \epsilon^2
	\\
	&\qquad\leq
	4 C^2 
	\br{d \log\br{k} 
	+ \log\frac{\cN_{1/k}(\cV)}{p}}
	+ 8 
	\tag{setting $\epsilon=1/k$}
\end{align*}
Combining \cref{eq:valcon_err} with the inequalities from the last two displays gives;
\begin{align*}
	\norm{\hat v_h - v_h}_{\Lambda_{h}}^2
	&\leq  
	16 C^2 
	\br{2d \log\br{k} 
	+ \log\frac{\cN_{1/k}(\cV)}{p}}
	\\
%	&\leq  
%	4 C^2\br{ \log (k+1)
%	+ 4C^2 \log\frac{\cN_\epsilon(\cV)}{p} 
%	+ 8
%	+ 2dC^2	}
%	\\
	\implies 
	\norm{\hat v_h - v_h}_{\Lambda_{h}}
	&\leq 
	4 C 
	\sqrt{2d \log\br{k} 
	+ \log\frac{\cN_{1/k}(\cV)}{p}}
	.
\end{align*}
Finally,
\begin{align*}
	\av{\br[b]{\P_h - \overline \P_{h}} V(s, a)}
	=\av{\phi(s, a)\T \br{\hat v_h - v_h}}
	\leq \norm{\phi(s, a)}_{\Lambda_h^{-1}}
	\norm{\hat v_h - v_h}_{\Lambda_h},
\end{align*}
which completes the proof after plugging in the bound from the previous display.
\end{proof}

\subsection{Value and policy classes}
\label{sec:value_policy_classes}
Given an input parameter $\beta$ given to \cref{alg:oppo_linmdp}, we consider the softmax policy class as defined below:
\begin{align}
	\cY(D_w, \lambda_-, \lambda_+, J_{\max}) &\eqq \cbr{
		y(\cdot; w, W_{1:J})
		\mid 
		\norm{w} \leq D_w,
		\lambda_{-} I \preceq W_j \preceq \lambda_{+} I,
		J \leq J_{\max}
	},
	\nonumber \\
	&\quad \text{where }
	y(x; w, W_{1:J})
	\eqq x\T w + \sum_{j=1}^J \norm{x}_{W_j};
	\nonumber \\
	\Pi &\eqq \cbr{\pi(\cdot|\cdot; y) 
	\mid y \in \cY(3dHK^2, K^{-2}, \beta^2K^2, 2d \log K)},
	\nonumber \\
	\label{def:policy_class}
	&\quad \text{where } 
	\pi( a | s; y) \eqq \frac{e^{y(\phi(s, a))}}
	{\sum_{b}e^{y (\phi(s, b))}}.
\end{align}
We further consider the following class of empirical restricted (Q-)functions:
\begin{align}
	\widetilde Q^{\circ} (s, a; w, W, \cZ) 
	&\eqq 
	\I\cbr{s\in \cZ}\br{\phi(s, a)\T w - \sqrt{\phi(s, a)\T W \phi(s, a)}},
	\nonumber \\
	\label{def:Q_class}
	\widetilde \cQ^{\circ}(\cZ, C)
	&\eqq \cbr{\widetilde Q^\circ (\cdot, \cdot; w, W, \cZ) 
	\mid \norm{w}_2 \leq 2 d H K, \norm{W}_2 \leq \beta^2,
	\norm{\widetilde Q^\circ (\cdot, \cdot; w, W, \cZ)}_\infty \leq C},
\end{align}
and their corresponding value functions:
\begin{align*}
	\widetilde V (s; \pi_h, \widetilde Q^\circ) 
	\eqq 
	\abr{\pi_h(\cdot|s), \widetilde Q^\circ(s, \cdot)}
	.
\end{align*}
Now define the following empirical restricted value function class:
\begin{align}
	\label{def:value_class}
	\widetilde \cV^\circ(\cZ, C) 
	&= \cbr{ \widetilde V (\cdot; \pi_h, \widetilde Q^\circ) \colon \cS \to \R
	\mid 
	\widetilde Q^\circ \in \cQ^\circ(\cZ, C), \pi_h \in \Pi}
	.
\end{align}
The following lemma (of which the proof is deferred to \cref{sec:empirical_value_covering}) provides the bound on the covering number of the function class defined in \cref{def:value_class} above.

\begin{lemma}
	\label{lem:empirical_value_covering}
		There exists a universal constant $c_\cN$, such that for any 
		$\nu > 0$, $\cZ \subseteq \cS$,
		\begin{align*}
			\log \cN_\nu(\widetilde \cV^\circ(\cZ, C))
			\leq c_{\cN} d^3 \log\br{\beta C K /\nu }.
		\end{align*}
	\end{lemma}

\subsection[The good event lemma]{Proof of \cref{lem:good_event}}
\label{sec:goodevent}
In this section we provide the full technical details for the analysis of the good event \cref{lem:good_event}. 
The core part of the argument establishes the confidence bounds for the regression step in spite of the absence of the truncation.
To begin, we first define an additional success event; the concentration of least squares errors uniformly over the class of empirical value functions (recall the function class $\widetilde \cV^\circ$ is defined in \cref{def:value_class}).
\begin{align}
    \forall k\geq K_0, h;
    \forall 
    V_{h+1} \in \widetilde \cV^\circ(\cZ_{h+1}, 2H);
    \forall s, a:
    \;
    \av{\br[b]{\P_h - \widehat \P_{h}^{k}} V_{h+1}(s, a)}
        \leq 
        (\beta/2)
        \norm{\phi(s, a)}_{\Lambda_{k,h}^{-1}},
    \tag{$\cE^{\rm uls}$}
    \label{def:goodevent_uls}
\end{align} 
The core argument pertaining to the regression errors proceeds as follows.
\begin{enumerate}
    \item \cref{lem:goodevent_rfw}, 
    establishes the success probability of \ref{def:goodevent_rfw}. For the most part this follows from the guarantees of the $\CoverTraj$ algorithm developed in the prior work of \citet{wagenmaker2022reward}.
    \item \cref{lem:goodevent_uls} establishes the success probability of \ref{def:goodevent_uls}; ensuring concentration of the regression errors w.r.t.~the value function classes $\widetilde \cV^\circ(\cZ_{h+1}, 2H)$.
    \item Given that \ref{def:goodevent_rfw} and \ref{def:goodevent_uls} both hold,
    \cref{lem:goodevent_vbu} provides, using a careful inductive argument, that the value functions estimated in \cref{alg:oppo_linmdp} are contained in the function class $\widetilde \cV^\circ(\cZ_{h+1}, 2H)$. Thus, \ref{def:goodevent_qbd} and \ref{def:goodevent_vbu} hold.
\end{enumerate}

\begin{lemma}[success of \ref{def:goodevent_bon}]
	\label{lem:goodevent_bonus}
		For any $\delta > 0$, we have that with probability $\geq 1-\delta$, for all $h$:
			\begin{align*}
				\sum_{k=1}^K  \E_{\mu_h^k}\sbr{
					\norm[b]{\phi(s_h, a_h)}_{\Lambda_{k,h}^{-1}}
				}
				&\leq 2 \sum_{k=1}^K
					\norm[b]{\phi(s_h^k, a_h^k)}_{\Lambda_{k,h}^{-1}}
					+
					4 \log \frac{4 K H}{\delta}.
			\end{align*}
	\end{lemma}
	\begin{proof}
		Denote $X_k = \norm[b]{\phi(s_h^k, a_h^k)}_{\Lambda_{k,h}^{-1}}$, and 
		recall the definition of $\cF^k$ in \cref{eq:def:filtration}. 
        Then 
		$X_k$ is $\cF^k$ measurable, and 
		\begin{align*}
			\E_{\mu_h^k}\sbr{
					\norm[b]{\phi(s_h, a_h)}_{\Lambda_{k,h}^{-1}}
				}
			= \E\sbr{X_k \mid \cF^{k-1}} 
			.
		\end{align*}
		In addition, by the definition of $\Lambda_{k, h}=I + \sum_{i\in \cD_h^k}\phi(s_h^i, a_h^i)\phi(s_h^i, a_h^i)\T$ in \cref{alg:oppo_linmdp}, and by the assumption that $\norm{\phi(s, a)}\leq 1$ (\cref{def:linmdp}), we have that $0 \leq X_k \leq 1$. Thus
		by \cref{lem:nameless_concentration} and the union bound, we have that w.p.~$1-\delta$, for all $h \in [H]$:
		\begin{align*}
			\sum_{k=1}^K \E\sbr{X_k \mid \cF^{k-1}}
			\leq 2 \sum_{k=1}^K X_k + \log \frac{2 K H}{(\delta/KH)}
			\leq 2 \sum_{k=1}^K
				\norm[b]{\phi(s_h^k, a_h^k)}_{\Lambda_{k, h}^{-1}}
				+
				4 \log \frac{2 K H}{\delta},
		\end{align*}
		which completes the proof.
	\end{proof}

\begin{lemma}[success of \ref{def:goodevent_uls}]
	\label{lem:goodevent_uls}
	There exists a constant $c_\beta > 0$, such that when running \cref{alg:oppo_linmdp} with 
	$\beta \geq 2 c_\beta d H \log(d H K/\delta)$, 
	we have that the event \ref{def:goodevent_uls}
 holds with probability $\geq 1-\delta$.
\end{lemma}
\begin{proof}
	Let $h, k \in [H] \times \cbr{K_0, \ldots, K}$, and
	note that since we define the regression solution using
	\begin{align*}
	\cD_h^k = \cD_h^0 \cup \cbr{(s_h^i, a_h^i, s_{h+1}^i)}_{i=K_0}^{k-1}
		,
	\end{align*}
	this dataset is independent of the known states set $\cZ_{h+1}$ which is defined using $\cD_{h+1}^0$. Indeed, this is because $\cbr{\cD_{h}^0}_{h\in [H]}$ were generated by \emph{independent runs} of $\CoverTraj$ in \cref{alg:reward_free}.
	Hence, the value class $\widetilde \cV^\circ(\cZ_{h+1}, 2H)$ is independent of $\cD_h^k$, and we may apply
\cref{lem:value_concentration_base} to obtain that w.p. $\geq 1- \delta'$, 
	\begin{align*}
		\forall 
		V_{h+1} \in \widetilde \cV^\circ(\cZ_{h+1}, 2H),
			\forall s\in \cS, a\in \cA:
		\;
		\av{\br[b]{\P_h - \widehat \P_{h}^{k}} V_{h+1}(s, a)}
			\leq 
			\hat \beta
			\norm{\phi(s, a)}_{\Lambda_{k,h}^{-1}},
	\end{align*}
	where
	\begin{align*}
		\hat \beta = \br{8 H \sqrt{2d \log k
		+ \log\frac{\cN_{1/k}(\widetilde \cV^\circ(\cZ_{h+1}, 2H))}{\delta'}}}
		.
	\end{align*}
        By \cref{lem:empirical_value_covering}, we have
	\begin{align*}
		\log \cN_{1/K}(\widetilde \cV^\circ(\cZ_{h+1}, 2H)) 
		\leq c d^3 \log\br{\beta H K},
	\end{align*}
        for some universal constant $c$, which implies that
        \begin{align*}
            \hat \beta
            \leq 
		c' d^{3/2} H \log(\beta H K/\delta'),
        \end{align*}
	for a suitable constant $c' > 0$. Setting $\delta'=\delta/ K H$, we now have by the union bound that,
	\begin{align*}
		\forall k, h;
		\forall 
		&V_{h+1} \in \widetilde \cV^\circ(\cZ_{h+1}, 2H);
		\forall s, a:
		\\
		&\av{\br[b]{\P_h - \widehat \P_{h}^{k}} V_{h+1}(s, a)}
			\leq 
			2c' d^{3/2} H \log(\beta H K/\delta)
			\norm{\phi(s, a)}_{\Lambda_{k,h}^{-1}}.
	\end{align*}
	Finally, 
	by \cref{lem:beta_recursion_choice}, 
	for a suitable $c_\beta > 0$ we have
	\begin{align*}
		\beta/2
		\geq c_\beta d^{3/2} H \log(d H K/\delta)
		\geq 2c' d^{3/2} H \log(\beta H K/\delta),
	\end{align*}
	which completes the proof.
	
\end{proof}

\begin{lemma} [success of \ref{def:goodevent_sle}]
    \label{lem:goodevent_lse}
    Consider running \cref{alg:oppo_linmdp} in the stochastic case with bandit feedback, with 
    $\beta\geq 2 c_\beta d H \log(d H K/\delta)$ as specified by \cref{lem:goodevent_uls}.
    Then, we have that the event \ref{def:goodevent_sle} holds w.p.~$1-\delta$.
\end{lemma}
\begin{proof}
    For a given $k, h$, we have
    \begin{align}
			\forall s, a: \;
			\av{\widehat \l_h^k(s, a) - \l_h^k(s, a)}
			=
			\av{\phi(s, a)\T \br{\widehat g_{k, h} - g_{k, h}}}
			\leq
			\norm{\phi(s, a)}_{\Lambda_{k,h}^{-1}}
                \norm{\widehat g_{k, h} - g_{k, h}}_{\Lambda_{k, h}
                }.
                	\label{eq:sle_eq1}
    \end{align}
    Following the same algebraic argument as that given in \cref{lem:value_concentration_base}, we have
    \begin{align*}
        \norm{\widehat g_{k, h} - g_{k, h}}_{\Lambda_{k, h}
        }
        &=
        \norm{\Lambda_{k, h}^{-1} 
        \sum_{i = 1}^{k-1} \phi(s_h^i, a_h^i) \br{ 
		\l_h^i(s_{h}^i, a_h^i) 
            - \l_h(s_{h}^i, a_h^i)
	}
	- \Lambda_{k, h}^{-1} g_{k, h}
        }_{\Lambda_{k, h}}
        \\
        &\leq 
        \norm{
        \sum_{i = 1}^{k-1} \phi(s_h^i, a_h^i) \br{ 
		\l_h^i(s_{h}^i, a_h^i) 
            - \l_h(s_{h}^i, a_h^i)
	}}_{\Lambda_{k, h}^{-1}}
        + 
        \norm{g_{k, h}}_{\Lambda_{k, h}^{-1}}.
    \end{align*}
    By \cref{lem:ols_concentration} (the application of which is legitimate due to \cref{assume:linmdp_stochastic}) and the union bound,
    for any $\delta > 0$
    the first term above is bounded by $\sqrt{4 d\log (H K/\delta)}$ for all $k,h$, while the second term is bounded a.s. by $\sqrt d$ owed to the assumption in \cref{def:linmdp} and that $\Lambda_{k, h}^{-1} \preceq I$. Concluding, we have w.p. $\geq 1-\delta$, for all $k, h$:
    \begin{align*}
     \norm{\widehat g_{k, h} - g_{k, h}}_{\Lambda_{k, h}
        }\leq 
        \sqrt{4 d\log (H K/\delta)} + \sqrt d
        \leq \beta/2.
    \end{align*}
    The proof is complete after plugging the above inequality into \cref{eq:sle_eq1}.
\end{proof}

\begin{lemma}
\label{lem:alg_weights_bound}
	Let $\cD_h=\cbr{(s_h^i, a_h^i)}_{i\in [k]}$, and $\Lambda_h \eqq I + \sum_{i\in \cD_h} \phi(s_h^i, a_h^i)\phi(s_h^i, a_h^i)\T$.
	Then,
	\begin{align*}
		\norm[Bg]{\Lambda_h^{-1}
		\sum_{i\in \cD_h} 
			\phi(s_h^i, a_h^i)}_2 \leq \sqrt{d k}.
	\end{align*}
\end{lemma}
\begin{proof}
	Follows from the exact same argument as in \citet{jin2020provably}, Lemma B.2.
\end{proof}

\begin{lemma}
\label{lem:policy_value_induction}
	Let $K_0 \leq \tau < K$, $h\in[H]$, and 
	assume $\widetilde V_{h+1}^{k;\circ}\in \widetilde \cV^\circ (\cZ_{h}, 2 H)$ for all $k\in \cbr{K_0, \ldots, \tau-1}$.
	Then $\pi_h^{\tau+1} \in \Pi$, where $\Pi$ is  defined in \cref{def:policy_class}, and $\pi_h^{\tau+1}$ is a mirror descent step from $\pi_h^\tau$ as defined in \cref{alg:oppo_linmdp}.
\end{lemma}
\begin{proof}
	By the definition of the OMD update step in \cref{alg:oppo_linmdp}, we have for any $a, s$;
	\begin{align*}
		\pi_h^{\tau+1}(a|s)
		=
		\frac{e^{-\eta \sum_{k=K_0}^\tau \widetilde Q_h^k(s, a)}}
		{\sum_{a'}e^{-\eta \sum_{k=K_0}^\tau \widetilde Q_h^k(s, a')}}.
	\end{align*}
        
	In addition, by the definition of the estimated Q-functions $\widetilde Q_h^k$ in \cref{alg:oppo_linmdp}, we have;
	\begin{align*}
		-\eta \sum_{k=K_0}^\tau \widetilde Q_h^k(s, a)
		&=
		-\eta \sum_{k=K_0}^\tau \br{\widehat \l_h^k(s, a) + \widehat \P_h^k \widetilde V_{h+1}^{k;\circ}(s, a)}
			-\hat b_h^k(s, a)
		\\
		&= 
		-\eta \sum_{k=K_0}^\tau \phi(s, a)\T  \br[b]{\widehat g_{k, h} +\widehat v_h^k}
		+ \eta \sum_{k=K_0}^\tau \hat b_h^k(s, a)
		\\
		&= 
		-\eta \sum_{k=K_0}^\tau \phi(s, a)\T  \br[b]{\widehat g_{k, h} +\widehat v_h^k}
		+ \eta \beta \sum_{k=K_0}^\tau \norm{\phi(s, a)}_{\widehat \Lambda_{k, h}^{-1}}
		\\
		&= 
		\phi(s, a)\T \br{-\eta \sum_{k=K_0}^\tau  \widehat g_{k, h} +\widehat v_h^k}
		+ \eta \beta \sum_{j=1}^J (k_{j+1} - k_j) \norm{\phi(s, a)}_{\Lambda_{k_j, h}^{-1}},
	\end{align*}
	where $k_j$ are the episodes on which we update the bonus matrices $\widehat \Lambda_{k, h}$ in \cref{alg:oppo_linmdp}. 
	Now, since for all $K_0\leq k \leq K$ we have
        \begin{align*}
        \norm{\Lambda_{k,h}} = \norm{I + \sum_{i\in \cD_h^k}\phi(s_h^i, a_h^i)\phi(s_h^i, a_h^i)\T} \leq \norm{I} + \sum_{i\in \cD_h^k}\norm{\phi(s_h^i, a_h^i)\phi(s_h^i, a_h^i)\T}
        \leq 1 + K,
        \end{align*}
        and $I \preceq \Lambda_{k,h}$, it follows that
        \begin{align*}
		2^J\det \Lambda_{K_0, h}
		\leq \det \Lambda_{K, h}
            \leq \norm{\Lambda_{k,h}}^d
		\leq (K+1)^d,
	\end{align*}
	and $1 \leq \det \Lambda_{K_0, h}$.
        Thus it is implied that $J \leq d\log (K+1) \leq 2d\log K$. 
	In addition, 
		\begin{aligni*}
			\eta \beta (k_{j+1} - k_j) \norm{\phi(s, a)}_{\widehat \Lambda_{k_j, h}^{-1}}
			= \norm{\phi(s, a)}_{W_j}
		\end{aligni*}
		when we define
		\begin{align*}
			W_j = {\eta^2\beta^2(k_{j+1} - k_j)^2} \Lambda_{k_j, h}^{-1}
			, \text { and thus }
			\frac{1}{K^2} I
			\preceq W_j \preceq
			\beta^2 K^2 \Lambda_{k_j, h}^{-1}
			\preceq \beta^2 K^2 I
			.
		\end{align*}	
		Furthermore, in the adversarial case $\norm{\widehat g_{k, h}} = \norm{g_{k, h}} \leq \sqrt d$ by assumption (see \cref{def:linmdp}), and in the stochastic 
  case,
    \begin{align*}
			\norm{\widehat g_{k, h}}
			= \norm{ \Lambda_{k, h}^{-1} \sum_{i\in \cD_{h}^k}
			\phi(s_h^i, a_h^i)
                \l_h^i(s_h^i, a_h^i)
			}
			\leq \sqrt{d K},
		\end{align*}
		where the inequality follows from  \cref{lem:alg_weights_bound} and our assumption that $\av{\l_h^k(s_h^i, a_h^i)}\leq 1$.
        In addition, in both the stochastic and adversarial cases, we have
	\begin{align*}
			\norm{\widehat v_h^k}
			= \norm{ \Lambda_{k, h}^{-1} \sum_{i\in \cD_{h}^k}
			\phi(s_h^i, a_h^i)\widetilde V^{k;\circ}_{h+1} (s_{h+1}^i)
			}
			\leq 2 H\sqrt{d K},
		\end{align*}
		which follows again by  
		\cref{lem:alg_weights_bound},
		and our assumption that 
		$\widetilde V_{h+1}^{k;\circ}\in \widetilde \cV^\circ (\cZ_{h}, 2 H) \implies \norm{\widetilde V_{h+1}^{k;\circ}}_\infty\leq 2H$ for all $k\leq \tau$.
		Thus, $\norm{\widehat g_{k, h} + \widehat v_h^k}\leq 3H\sqrt{d K}$ for all $k\leq \tau$.
		Concluding, we have shown that
		\begin{align*}
			\pi_h^{\tau+1}(a|s)
			\propto
			\exp\br{
			\phi(s, a)\T w_h^\tau
			+ \sum_{j=1}^J \norm{\phi(s, a)}_{W_j}
			},
		\end{align*}
		where $\norm{w_h^\tau}\leq 3 d H K^2$ and 
		$K^{-2}I \preceq W_j \preceq \beta^2 K^2 I$, therefore $\pi_h^{\tau+1} \in \Pi$, as required.
\end{proof}

\begin{lemma}[success of \ref{def:goodevent_qbd} $\cup$ \ref{def:goodevent_vbu}]
	\label{lem:goodevent_vbu}
		Assume that the event 
    \ref{def:goodevent_rfw} $\cup$
    \ref{def:goodevent_uls} $\cup$
    \ref{def:goodevent_sle}
    holds.
		Then, we have that,
		\begin{align*}
            \forall k\geq K_0, h\in [H]:
		\widetilde Q_h^{k;\circ} \in \widetilde \cQ^\circ(\cZ_h, C_h),
			\widetilde V_{h}^{k;\circ} 
			\in \widetilde \cV^\circ (\cZ_{h}, C_{h}),
		\end{align*} 
		where $C_h \eqq \br{H-h+1}\br{1 + 2/H}$. Furthermore, we have that the event \ref{def:goodevent_vbu} holds, that is, 
        \begin{align}
			\forall k\geq K_0, h\in [H];
			\forall s, a: \;
			\av{\br{\P_h - \widehat \P_{h}^{k}} \widetilde V^{k;\circ}_{h+1}(s, a)}
			\leq
			(\beta/2)
			\norm{\phi(s, a)}_{\Lambda_{k,h}^{-1}}.
		\end{align}
	\end{lemma}
	\begin{proof} 
            We begin first by establishing simple bounds on the instantaneous loss estimates.
		For any $k\geq K_0$, we have in the adversarial case $\widehat \l_h^k(s, a) =  \l_h^k(s, a)$ for all $s,a,h,k$, so $\av[b]{\widehat \l_h^k(s, a)}\leq 1$ by the assumption in \cref{def:linmdp}.
In the stochastic case on the other hand,
for any $s, a\in \cZ_h \times \cA$, owed to our assumption that $\cE^{\rm sle}$ holds;
\begin{align*}
	\av{\widehat \l_h^k(s, a)}
        \leq 
        \av{\l_h^k(s, a)}
        + \av{\widehat \l_h^k(s, a) - \l_h^k(s, a)}
        \leq 
        1 + \beta\norm{\phi(s, a)}_{\Lambda_{k, h}^{-1}}
        .
		\end{align*}
		Furthermore, we have,
		\begin{align}
            \label{eq:bon_bound}
        \forall h\in [H], s,a\in \cZ_h\times \cA:\;
			\norm{\phi(s, a)}_{\Lambda_{k, h}^{-1}}
			\leq 
			\norm{\phi(s, a)}_{\widehat \Lambda_{k, h}^{-1}}
			\leq 
			\norm{\phi(s, a)}_{ \Lambda_{0, h}^{-1}}
			\leq 1/(2 \beta H),
		\end{align}
		where the last inequality follows from the definition of $\cZ_h$,
    thus we obtain $\beta\norm{\phi(s, a)}_{\Lambda_{k, h}^{-1}}\leq 1/(2H)$.
	To conclude, in both the stochastic and adversarial cases we have:
  \begin{align}
        \label{eq:simple_loss_bound}
      \forall k\geq K_0,h\in [H], s\in \cZ_h, a\in \cA;
      \;\;
      \av{\widehat \l_h^k(s, a)}\leq 1+1/(2H).
  \end{align}
  
  The rest of the proof proceeds by an inductive argument as follows.
  Fix $K_0 \leq k \leq K$, and assume we have already proved the claim for all $k',h \in  \cbr{K_0, \ldots, k-1} \times [H]$.
		We will now establish the claim for episode $k$ by induction on $h=H, \ldots, 1$.

		\paragraph{Base case $h=H$:} Here, we have 
		\begin{align*}
			\av{\widetilde Q_H^k(s, a)} 
            = \av{\widehat \l_H^k(s, a) - \hat b_h^k(s, a) 
            }
		\leq 1+1/(2H) + \beta \norm{\phi(s, a)}_{\widehat \Lambda_{k, h}^{-1}}
			\leq 1+1/H,
		\end{align*}
		where the first inequality follows from \cref{eq:simple_loss_bound},
		and the last inequality from  \cref{eq:bon_bound}.
		Thus, we obtain 
        $\widetilde Q_H^{k;\circ} \in \widetilde \cQ^\circ(\cZ_H, 1+1/H) \subset \widetilde \cQ^\circ(\cZ_H, C_H)$. Further, since $\widetilde V_{H+1}^{k';\circ} \equiv 0$ for any $k'\in [K]$, we may apply \cref{lem:policy_value_induction} which ensures $\pi_H^k \in \Pi$. Thus, it also follows that $\widetilde V_{H}^{k;\circ} (s)\in \widetilde V^\circ(\cZ_H, C_H)$.

		\paragraph{Inductive step:}
		Let $h<H$ and assume 
		$\widetilde V_{h+1}^{k;\circ} \in \widetilde \cV^\circ(\cZ_{h+1}, C_{h+1})$.
		For $s\in \cZ_h, a\in \cA$, we have;
          \begin{align*}
			\av{\widetilde Q_{h}^{k;\circ}(s, a)}
			&= 
			\av{\widehat \l_h^k(s, a)
			+ \widehat \P_h^k \widetilde V_{h+1}^{k;\circ} (s, a)
			- \hat b_h^k(s, a)}
			\\
			&=
			\av{\widehat \l_h^k(s, a)
			+ \P_h \widetilde V_{h+1}^{k;\circ} (s, a)
			+ (\widehat \P_h^{k} - \P_h) \widetilde V_{h+1}^{k;\circ} (s, a)
			- \hat b_h^k(s, a)}
			\\
			&\leq 
			1+1/H + C_{h+1} 
			+ \beta\norm{\phi(s, a)}_{\Lambda_{k, h}^{-1}}
			+ \beta\norm{\phi(s, a)}_{\widehat \Lambda_{k, h}^{-1}}
			,
		\end{align*}
		where the last inequality follows from \cref{eq:simple_loss_bound}, the inductive hypothesis, and by the assumption that $\cE^{\rm uls}$ holds.
        Applying \cref{eq:bon_bound} again,
	this implies that the empirical Q is well bounded on the known states;
		\begin{align*}
			\av{\widetilde Q_{h}^{k;\circ}(s, a)}
			\leq 
			1+ 1/H + C_{h+1} + 1/H
			= C_h.
		\end{align*}
		In addition, for any $s, a \in \cS\times \cA$;
		\begin{align*}
			\widetilde Q_{h}^{k}(s, a)
			=
			\phi(s, a)\T \widehat g_{k, h}
			+ \widehat \P_h^k \widetilde V_{h+1}^{k;\circ} (s, a)
			- \hat b_h^k(s, a)
			= 
			\phi(s, a)\T \br{\widehat g_{k, h} + \widehat v_h^k}
			-  \beta \norm{\phi(s, a)}_{\widehat \Lambda_{k, h}^{-1}},
		\end{align*}
		Further, as argued in the proof of \cref{lem:policy_value_induction}, by \cref{lem:alg_weights_bound} and our assumption that $\norm{\widetilde V_{h+1}^{k;\circ}}_\infty\leq 2H$, we have that
		\begin{align*}
			\norm{\widehat v_h^k}
			\leq 2 H\sqrt{d K},
   \quad \text { and }
                \norm{\widehat g_{k,h}}
			\leq \sqrt{d K}.
		\end{align*}
		In addition, 
		\begin{aligni*}
			\beta \norm{\phi(s, a)}_{\widehat \Lambda_{k, h}^{-1}}
			= \norm{\phi(s, a)}_{W}
		\end{aligni*}
		for 
		\begin{align*}
			W = \beta^2 \widehat \Lambda_{k, h}^{-1}
			, \text { and thus }
			\norm{W}
			= \beta^2 \norm{\widehat \Lambda_{k, h}^{-1}}\leq \beta^2
			.
		\end{align*}
		Therefore, we establish 
		that $\widetilde Q_h^{k;\circ} \in \widetilde \cQ^\circ(\cZ_h, C_h)$.
		Now, by our (first) inductive assumption that 
		$\widetilde V_{h+1}^{k';\circ} \in \widetilde \cV^\circ(\cZ_{h+1}, C_{h+1})$ for all $k' < k$,
		we may apply \cref{lem:policy_value_induction} to obtain that $\pi_h^k \in \Pi$.
		This immediately implies that
		$\widetilde V_{h}^{k;\circ} \in \widetilde \cV(\cZ_h, C_h)$, and completes the inductive argument.
	Finally, combined with our assumption that $\cE^{\rm uls}$ holds, this implies $\cE^{\rm vbu}$ holds, which completes the proof.
\end{proof}

We conclude this section with the proof of the good event \cref{lem:good_event}, which now follows easily by combining the above lemmas.
\begin{proof}[of \cref{lem:good_event}]
		By \cref{lem:goodevent_rfw,lem:goodevent_uls,lem:goodevent_bonus}, 
		and the union bound,
		we have that 
		$\cE^{\rm rfw}	\cup \cE^{\rm uls} \cup \cE^{\rm sle} \cup \cE^{\rm bon}$ holds w.p.$\geq 1-4\delta$.
		By \cref{lem:goodevent_vbu}, this now implies that $\cE^{\rm qbd} \cup \cE^{\rm vbu}$ holds as well, which completes the proof.
\end{proof}

\subsection{Covering of empirical value functions}
\label{sec:empirical_value_covering}

\begin{lemma}[Policy class is Lipschitz]
\label{lem:softmax_policy_Lipschitz}
	For any $\pi_h, 	\tilde \pi_h \in \Pi$, 
	$\pi_h(\cdot|\cdot) = \pi(\cdot|\cdot; y_h) , 
	\tilde \pi_h (\cdot|\cdot) = \tilde \pi (\cdot|\cdot; \tilde y_h) $, 
	parameterized by 
	$y_h(\cdot) = y_h(\cdot; w, W_{1:J})$,
	$\tilde y_h(\cdot) = y_h(\cdot; \tilde w, \tilde W_{1:J})$
	, we have for any $s\in \cS$:
	\begin{align*}
		\norm{\pi_h(\cdot|s) - \tilde \pi_h(\cdot|s) }_1
		\leq 6 K \sqrt{\norm{w - \tilde w}^2
		+ \sum_{j=1}^J \norm{W_j - \tilde W_j}^2}.
	\end{align*}
\end{lemma}
\begin{proof}
	We have, for any $x\in \R^d$,
	\begin{align*}
		\nabla_w y(x; w, W_{1:J}) &= x
		\\
		\nabla_{W_j} y(x; w, W_{1:J}) 
		&= \nabla_{W_j}\br{\sqrt {x\T W_j x}}
		= \frac{1}{2\sqrt {x\T W_j x}}x x\T.
	\end{align*} 
	Thus, considering $y(x; w, W_{1:J}) \in \cY$ implies 
	$K^{-2}I \preceq W_j$;
	\begin{align*}
		\norm{\nabla_{W_j} y(x; w, W_{1:J})}_F
		= \frac{1}{2\sqrt {x\T W_j x}}\norm{x x\T}_F
		= \frac{1}{2\sqrt {x\T W_j x}}\norm{x}^2
		\leq
		\frac{1}{2\sqrt{\lambda_{\min}(W_j)}\norm {x}}\norm{x}^2 
		\leq K \norm{x},
	\end{align*}
	which implies that when $\norm{x} \leq 1$,
	\begin{align*}
		\norm{\nabla_\theta y(x; \theta)} 
		= \sqrt {\norm{\nabla_w y(x; \theta)}^2 + \sum_{j=1}^J \norm{\nabla_{W_j} y(x; \theta)}_F^2}
		\leq \sqrt {\norm{x}^2 + K \sum_{j=1}^J \norm{x}^2}
		\leq 3\norm{x} K
		\leq 3 K.
	\end{align*}
	Hence, the parameterization $\theta \mapsto y(\cdot; \theta)$ is $(3 K)$-Lipschitz, and the result follows from \cref{lem:softmax_policy_Lipschitz_base}.
\end{proof}

The next lemma follows from similar arguments to those given in
\citet[][Lemma A.12]{wagenmaker2022instance}.
\begin{lemma}
\label{lem:softmax_policy_Lipschitz_base}
	Let $f_\theta \colon \R^d \to \R$ be any function parameterized by $\theta \in \R^p$, and assume the mapping
	$\theta \mapsto f_\theta(\phi(s, a))\in \R$ is $L$-Lipschitz for any $s, a$.
	Consider softmax policies 
	$\pi_h^\theta(\cdot|\cdot)=\pi_h(\cdot|\cdot;f_\theta), \;\pi_h^{\tilde \theta}(\cdot|\cdot)=\pi_h(\cdot|\cdot;f_{\tilde \theta}) \colon \cS \to \Delta_\cA$ as defined in \cref{def:policy_class}.
	Then, for any $\theta, \tilde \theta \in \R^p$, it holds that 	for any $s\in \cS$:
	\begin{align*} 
		\norm{\pi_h^\theta(\cdot|s) - \pi_h^{\tilde \theta}(\cdot|s) }_1
		\leq 2 L \norm[b]{\theta - \tilde \theta}_2.
	\end{align*}
\end{lemma}
\begin{proof}
	Let $v^s(\theta) \eqq f_\theta(\phi(s, \cdot)) \in \R^A$, and let 
	\begin{align*}
		\bJ v^s(\theta) \eqq 
		\begin{pmatrix}
			\nabla_\theta f_\theta(\phi(s, a_1))\T 
			\\
			\vdots
			\\
			\nabla_\theta f_\theta(\phi(s, a_A))\T 
		\end{pmatrix}
		\in \R^{A\times p}
	\end{align*} 
	denote the Jacobian of $v^s$ at $\theta\in \R^p$. 
	Then, we have by the chain rule:
	\begin{align*}
		\nabla_\theta  \pi_h^\theta(a|s)
		= \bJ v^s(\theta)\T  \nabla_u \br{\sigma(u)_a},
	\end{align*}
	where $u \eqq v^s(\theta)$ and $\sigma(u)_i=e^{u_i}/(\sum_{j}e^{u_j})$ denotes the softmax function. Combining with the softmax gradient
	$\nabla_u \br{\sigma(u)_a} = \sigma(u)_a\br{e_a - \sigma(u)}$, we get
	\begin{align*}
		\norm{\nabla_\theta  \pi_h^\theta(a|s)} 
		= \br{ \sigma(u)_a} \norm{\bJ v^s(\theta)\T \br{e_a - \sigma(u)}}
		\leq 2 \sigma(u)_a \max_a \norm{\nabla_\theta f_\theta(\phi(s, a))}
		\leq 2 L \pi_h^\theta(a|s)
		,
	\end{align*}
	where the last inequality uses our Lipschitz assumption and that $\sigma(u)_a = \pi^\theta_h(a|s)$.
	Now, by the mean-value theorem, we get that for some $\theta'\in [\theta, \tilde \theta]$,
	\begin{align*}
		\av{\pi_h^\theta(a|s) - \pi_h^{\tilde \theta}(a|s)}
            = 
            \av{\nabla \pi_h^{\theta'}(a|s)\br[b]{\theta - \tilde \theta}}
		\leq 2 L \pi_h^{\theta'}(a|s)\norm[b]{\theta - \tilde \theta}_2,
	\end{align*}
	which implies
	\begin{align*}
		\norm{\pi_h^\theta(\cdot|s) - \pi_h^{\tilde \theta}(\cdot|s)}_1
		\leq 2 L \norm[b]{\theta - \tilde \theta}_2,
	\end{align*}
	and completes the proof.
\end{proof}

	\begin{proof}[of \cref{lem:empirical_value_covering}]
		Let $\pi_h, \pi_h' \in \Pi$ be parameterized by 
		$\pi_h(\cdot|\cdot) = \pi(\cdot|\cdot; y_h), 
	 \pi_h' (\cdot|\cdot) = \pi (\cdot|\cdot; y_h') $, where 
	$y_h(\cdot) = y(\cdot; w, W_{1:J})$,
	$y_h'(\cdot) = y(\cdot; w', W_{1:J}')$,
		and consider
		$q, q'\in \widetilde \cQ^{\circ}(\cZ, C)$. For any $s\in \cZ$,
		we have;
		 \begin{align*}
			 \av{\widetilde V (s; \pi_h, q) - \widetilde V (s; \pi_h', q')}
			 \leq 
			 \av{\widetilde V(s; \pi_h, q) - \widetilde V(s; \pi_h, q')}
			 +
			 \av{\widetilde V(s; \pi_h, q') - \widetilde V(s; \pi_h', q')}.
		 \end{align*}
		 For the first term,
		 \begin{align}
			\av{\widetilde V(s; \pi_h, q) - \widetilde V(s; \pi_h, q')}
			 &\leq \max_a\cbr{
				 \av{\phi(s, a)\T\br{ w- w'}}
				 + \sqrt{\av{\phi(s, a)\T \br{W - W'} \phi(s, a)}}
			 }
			 \nonumber \\
			 &\leq \norm{w - w'} + \sqrt{\norm[b]{W - W'}}
            \label{eq:covering_term1}.
		 \end{align}
		 For the second term,
		 \begin{align*}
			\av{\widetilde V(s; \pi_h, q') - \widetilde V(s; \pi_h', q')}
			 &= \av{\abr{
				 \pi_h(\cdot|s) - \pi_h'(\cdot|s)
				 ,
				 q'(s, \cdot)
			 }}
			 \\
			 &\leq C \norm{\pi_h(\cdot|s) - \tilde \pi_h(\cdot|s)}_1 
			 \\
			 &\leq 6 C K \sqrt{\norm{w - w'}^2
			+ \sum_{j=1}^J \norm{W_j - W_j'}^2}
			\\
			 &\leq 6 C K \br{\norm{w - w'}
			+ \sum_{j=1}^J \norm{W_j - W_j'}},
		 \end{align*}
		 where the last inequality follows from \cref{lem:softmax_policy_Lipschitz}.
		 As per \cref{def:policy_class}, we have that $w, w'\in \cB^d(3dHK^2)$, $J\leq 2 d\log K$, and $W_j, W_j' \in \cB^{d^2}(\sqrt d \beta^2 K^2)$ for all $j\leq J$, where this last claim follows since the Frobenius norm of any matrix is larger than its spectral norm by a factor of at most $\sqrt d$.
		 Thus, for simplicity, we consider covering the larger set given by $E \eqq \cB^p(4dH\beta^2K^2)$ and $p=4d^3\log K$.
		 By \cref{lem:covering_l2_ball},
		 given any $\nu$, we have a $(\nu_1 = \nu/(12 C K))$-covering with cardinality $\leq (1 + (4dH\beta^2K^2)*12 C K/\nu)^p = (1 + 48dCH\beta^2K^3/\nu)^p$.
		 
		 Similarly, we $\nu/4$ construct a cover corresponding to each of the terms in \cref{eq:covering_term1} with sets of cardinality 
		 $(1+64\beta^2/\nu^2)^{d^2}$ and $(1+16dHK^2/\nu)^{d}$.
		 This gives,
		\begin{align*}
			 \log \cN_\nu(\widetilde \cV^\circ(\cZ, C))
			 &\leq p \log \br{1 + 48dCH\beta^2K^3/\nu} 
			 + 2 d^2 \log \br{1+ 64 \beta/\nu}
			 + d \log (1+ 16 d H K^2/\nu)
			 \\
			 &\leq
			 c_{\cN} d^3 \log\br{\beta C K /\nu }
			 ,
		 \end{align*}
		 for an appropriate constant $c_{\cN}$, which completes the proof. 
	\end{proof}

\subsection[Proof of regret decomposition]{Proof of \cref{lem:regret_decomposition}}

\begin{proof}[of \cref{lem:regret_decomposition}]
	For any $k$, we have by \cref{lem:extended_value_diff};
	\begin{align*}
		V_1^{k, \pi^k} - V_1^{k, \pi^\star}
		&=
		V_1^{k, \pi^k} - \widetilde V_1^{k;\circ}  
		+ \widetilde V_1^{k;\circ} - V_1^{k, \pi^\star}
		\\
		&=
		\sum_{h=1}^H \E_{\mu_h^k}\sbr{
			\l_h^k(s_h, a_h) + \P_h \widetilde V_{h+1}^{k;\circ}(s_h, a_h)
			- \widetilde Q^{k;\circ}_h(s_h, a_h)
		}
		\\
		&\quad + 
		\sum_{h=1}^H \E_{\mu_h^\star}\sbr{
			\abr{\widetilde Q^{k;\circ}_h(s_h, \cdot), 
			\pi_h^k(\cdot|s_h) - \pi_h^\star(\cdot|s_h)
			}
		}
		\\
		&\quad+
		\sum_{h=1}^H  \E_{\mu_h^\star}\sbr{
			\widetilde Q^{k;\circ}_h(s_h, a_h) 
			- \l^k_h(s_h, a_h)
				- \P_h \widetilde V_{h+1}^{k;\circ}(s_h, a_h)
		}.
	\end{align*}
	Now, note that for any $s\in \cZ_h, a\in \cA$;
	\begin{align*}
		\widetilde Q_h^{k;\circ}(s, a)
		=
		\phi(s, a)\T \widehat g_{k, h}
		+ \widehat \P_h^k \widetilde V_h^{k;\circ}(s, a)
		-\hat b_h^k(s, a),
	\end{align*}
	thus
	\begin{align*}
		\l_h^k(s, a) + \P_h \widetilde V_{h+1}^{k;\circ}(s, a)
			- \widetilde Q^{k;\circ}_h(s, a)
		= 
		\phi(s, a)\T\br{g_{k, h} - \widehat g_{k, h}}
		+ 
		\br{\P_h - \widehat \P_h^k} \widetilde V_{h+1}^{k;\circ} (s, a)
		+ \hat b_h^k(s, a).
	\end{align*}
	In addition, by the good event, specifically \ref{def:goodevent_qbd}, and the assumption that the instantaneous loss is $\in[-1, 1]$, we have 
	for any $s\notin \cZ_h, a\in \cA$:
	\begin{align*}
		\av{
			\l_h^k(s, a) + \P_h \widetilde V_{h+1}^{k;\circ}(s, a)
			- \widetilde Q^{k;\circ}_h(s, a)
		}
		= 
		\av{\l_h^k(s, a) + \P_h \widetilde V_{h+1}^{k;\circ}(s, a)}
		\leq 1 + H + 2 \leq 2H.
	\end{align*}
	Thus by the law of total expectation,
	\begin{align*}
		\E_{\mu_h^k}&\sbr{
			\l_h^k(s_h, a_h) + \P_h \widetilde V_{h+1}^{k;\circ}(s_h, a_h)
			- \widetilde Q^{k;\circ}_h(s_h, a_h)
		}
		\\
		&\leq
		\E_{\mu_h^k}\sbr{
			\l_h^k(s_h, a_h) - \widehat \l_h^k(s_h, a_h)
			+ \br{\P_h - \widehat \P_h^k} \widetilde V_{h+1}^{k;\circ} (s_h, a_h)
			+ \hat b_h^k(s, a) 
			\mid s_h \in \cZ_h
		}
		+ 2 \epsilon_{\rm cov} H,
	\end{align*}
	where the inequality follows since the good event \ref{def:goodevent_rfw} implies $\mu_h^k(\cS \setminus \cZ_h) \leq \epsilon_{\rm cov}$, and for similar reasons;
	\begin{align*}
		\E_{\mu_h^\star}&\sbr{
			\widetilde Q^{k;\circ}_h(s_h, a_h) 
			- \l^k_h(s_h, a_h)
				- \P_h \widetilde V_{h+1}^{k;\circ}(s_h, a_h)
		}
		\\
		&\leq
		\E_{\mu_h^\star}\sbr{
			\widehat \l_h^k(s_h, a_h) - \l_h^k(s_h, a_h)
			+
			\br{\widehat \P_h^k - \P_h} \widetilde V_{h+1}^{k;\circ} (s_h, a_h)
			-\hat b_h^k(s, a)
			\mid s_h \in \cZ_h
		}
		+ 2 \epsilon_{\rm cov} H.
	\end{align*}
	Finally, again by the law of total expectation and definition of the restricted Q-function;
	\begin{align*}
		\E_{s_h \sim \mu_h^\star}\sbr{
			\abr{\widetilde Q^{k;\circ}_h(s_h, \cdot), 
			\pi_h^k(\cdot|s_h) - \pi_h^\star(\cdot|s_h)
			}
		}
		\leq 
		\E_{s_h \sim \mu_h^\star}\sbr{
			\abr{\widetilde Q^{k}_h(s_h, \cdot), 
			\pi_h^k(\cdot|s_h) - \pi_h^\star(\cdot|s_h)
			}
			\mid s_h \in \cZ_h
		}.
	\end{align*}
	Combining the last three displays with the first equation and summing over $k=K_0, \ldots, K$ and $h\in [H]$ completes the proof.
\end{proof}

\begin{lemma}
\label{lem:bonus_doubling}
	Upon execution of \cref{alg:oppo_linmdp}, for all $k, h$ it holds that
	\begin{align*}
		\forall u\in \R^d; \quad 
		\norm{u}_{\Lambda_{k, h}^{-1}}
		\leq 
		\norm{u}_{\widehat \Lambda_{k, h}^{-1}}
		\leq 
		\sqrt 2\norm{u}_{ \Lambda_{k, h}^{-1}}
		.
	\end{align*}
\end{lemma}
\begin{proof}
	By definition, we have at all times $\widehat  \Lambda_{k, h} \preceq \Lambda_{k, h} $
	and 
	$\det \Lambda_{k, h} \leq 2 \det \widehat \Lambda_{k, h}$.
	Therefore, $  \Lambda_{k, h}^{-1} \preceq \widehat \Lambda_{k, h}^{-1} $
	and 
	$\frac{\det \widehat \Lambda_{k, h}^{-1}}{\det \Lambda_{k, h}^{-1}} \leq 2 $.
	Now, by \cref{lem:bonus_doubling_base}, we have
	\begin{align*}
		\Lambda_{k, h}^{-1} 
		\preceq \widehat \Lambda_{k, h}^{-1} 
		\preceq 2\Lambda_{k, h}^{-1},
	\end{align*}
	which completes the proof.
\end{proof}

\section[Proof of reward free warmup]{Proof of \cref{lem:goodevent_rfw}}
\label{sec:goodevent_rfw}
In this section, we provide the technical details of the reward free algorithm guarantees. As mentioned, the algorithm is based on the work of \citet{wagenmaker2022reward} --- the basic guarantee we build upon is formally stated below and follows immediately from Theorem 2 and Corollary 2 in their work.
The bound on the number of episodes $T$ follows from plugging the guarantees of $\textsc{Force}$ \citep[Algorithm~1]{wagenmaker2022first} into the precise setting of $K_i$ given in the beginning of Appendix B of \cite{wagenmaker2022reward}.
\begin{theorem}[\citealp{wagenmaker2022reward}]
\label{thm:wagenmaker_cover_traj}
	The $\textsc{CoverTraj}$ algorithm \citep[Algorithm~4]{wagenmaker2022reward} when instantiated with $\textsc{Force}$ \citep[Algorithm~1]{wagenmaker2022first} enjoys the following guarantee.
	Given a sequence of tolerance parameters $\gamma_1, \ldots, \gamma_m > 0$ and $h\in [H]$, the algorithm interacts with the environment for $T$ steps, where
	\begin{align*}
		T \leq  	
			C \sum_{i=1}^m 2^i \max\cbr{
				\frac{d}{\gamma_i^2}\log \frac{2^i}{\gamma_i^2}, 
				d^4 H^3 m^3  \log^{7/2}\frac1\delta
			}
		, \quad C > 0 \text{ a constant,}
	\end{align*}
	and outputs 
	\begin{aligni*}
		\cbr[b]{\br[b]{\cX_{h, i}, \widetilde \cD_{h,i}, \widetilde \Lambda_{h,i}}}_{i=1}^m
	\end{aligni*}
	such that
 $\bigcupdot_{i=1}^{m+1} \cX_{h, i} = B_0^d(1)$ partitions the euclidean unit ball,
 $\widetilde \Lambda_{h,i} = I + \sum_{\tau\in \widetilde \cD_{h, i}} \phi(s_h^\tau, a_h^\tau) \phi(s_h^\tau, a_h^\tau)\T $, and
 with probability $1-\delta$, it holds that:
	\begin{align*}
		&\forall i\in [m], \;
		\phi\T \widetilde \Lambda_{h,i}^{-1}\phi \leq \gamma_i^2, 
		\; 
		\forall \phi \in \cX_{h, i};
		\\
            \text{and } &\forall i\in [m+1],
		\sup_{\pi}\cbr{\int_{\cS \times \cA} 
        \I\cbr{\phi(s, a) \in \cX_{h, i}}\mu_h^\pi(s, a) }\leq 2^{-i+1}
		.
	\end{align*}
\end{theorem}

\begin{lemma}
\label{lem:unknwon_states_weight}
	Assume $h\in [H], \epsilon, \delta > 0, \gamma_m \geq \cdots \geq \gamma_1 > 0$, and let
 $\cbr{\Lambda_{h, i}}_{i\in[m]}$ be the covariate matrices returned from \CoverTraj$(h, \delta, m=\log(1/\epsilon), \cbr{\gamma_i})$. Then under the assumption that the event from \cref{thm:wagenmaker_cover_traj} holds, 
	we have for any policy $\pi$ and $i\in[m]$:
	\begin{align*}
		\Pr_{s_h \sim \mu^\pi_h}\br{ \exists a \text{ s.t. } \norm{\phi(s_h, a)}_{ \Lambda_{h, i}^{-1}} > \gamma_m}
		\leq \epsilon.
	\end{align*} 
\end{lemma}
\begin{proof}
By \cref{thm:wagenmaker_cover_traj}, we have that the total probability density induced by any policy $\pi\in[H] \times \cS \to \Delta(\cA)$ on the last partition set $\cX_{h, m+1}$ is at most $2^{-m}=\epsilon$. In addition, since on each of the remaining partition sets $\cbr{\cX_{h, i}}_{i\in [m]}$ we have the guarantee that $\phi \in \cX_{h, i} \implies \norm{\phi}_{\Lambda_{h, i}^{-1}} \leq \gamma_i \leq \gamma_m$, it follows that,
\begin{align*}
	\forall \pi ;\quad 
    \Pr_{s_h, a_h \sim \mu_h^\pi}
	\br{\norm{\phi(s_h, a_h)}_{\Lambda_{h,i}^{-1}} > \gamma_m} 
 = 
 \Pr_{s_h, a_h \sim \mu_h^\pi}
	\br{\phi(s_h, a_h) \in \cX_{h, m+1}} 
	\leq \epsilon
	.
	\tag{2}
\end{align*}
Assume by contradiction that $\pi$ is a policy for which the statement of the theorem does not hold. Then
\begin{align*}
	\Pr_{s_h\sim\mu_h^\pi}\br{\exists a, \norm{\phi(s_h, a)}_{\Lambda_{h,i}^{-1}} 
	> \gamma_m} > \epsilon.
\end{align*}
But, if this happens, we can consider a transformed policy $\tilde \pi$;
that rolls into timestep $h$ with $\pi$, then takes (with probability $1$) the action $a$ that maximizes $\norm{\phi(s_h, a)}_{\Lambda_{h, i}^{-1}}$. 
Formally, $\tilde \pi_{h'} = \pi_{h'}$ for all $h'\neq h$, and $\tilde \pi_h(a | s) = \I\cbr[b]{a\in \argmax_{a'} \norm{\phi(s, a')}_{\Lambda_{h, i}^{-1}}}$.
This implies,
\begin{align*}
	\Pr_{s_h, a_h \sim \mu_h^{\tilde \pi}}
	\br{\norm{\phi(s_h, a_h)}_{\Lambda_h^{-1}} > \gamma_m} > \epsilon,
\end{align*}
thus reaching a contradiction which completes the proof.
\end{proof}

\begin{proof}[of \cref{lem:goodevent_rfw}]
For the episode count, in order to apply \cref{thm:wagenmaker_cover_traj}, first note that given $\beta = O(d^{3/2}H\log(d H K/\delta)),\epsilon_{\rm cov}\geq 1/K$, we have:
 \begin{align*}
     	\forall i: \frac{d}{\gamma_i^2}\log \frac{2^i}{\gamma_i^2}
	= O(d\beta^2 H^2\log (2^i\beta H))
	= O(d\beta^2 H^2 \log^2 (\beta H K))
	= O(d^4 H^4 \log^4(d H K/\delta)).
 \end{align*}
 In addition,
 \begin{align*}
     d^4 H^3 m^3 \log^{7/2}\frac1\delta
     = O \br{d^4 H^3 \log^3 K \log^{7/2}\frac1\delta}
     = O \br{d^4 H^3 \log^7\frac K\delta}
 \end{align*}
 Hence, we have that for every $h$, with $T_h$ denoting the number of episodes run by $\CoverTraj$, by \cref{thm:wagenmaker_cover_traj};
	\begin{align*}
		T_h = O\br{d^4 H^4  \log^{7}(d H K/\delta) \sum_{i=1}^m 2^i}
		= O\br{2^{m+1} d^4 H^4  \log^{7}(d H K/\delta)}
            = O\br{\frac{d^4 H^4}{\epsilon_{\rm cov}}  \log^{7}(d H K/\delta)}.
	\end{align*}
	Given that \cref{alg:reward_free} executes \CoverTraj~$H$ times, the claim follows.
	For the claim on the un-reachability of $\cS\setminus\cZ_h$,
	fix $h\in [H]$, and observe that by \cref{lem:unknwon_states_weight}, w.p. $1-\delta/H$, for any $\pi$;
	\begin{align*}
		\Pr_{s_h \sim \mu_h^\pi}\br{s_h \notin \cZ_h}
		= \Pr_{s_h \sim \mu_h^\pi}\br{\exists a \text{ s.t. } 
		\norm{\phi(s_h, a)}_{\Lambda_{0,h}^{-1}}> \gamma_m}
		\leq \epsilon_{\rm cov},
	\end{align*}
	where in the inequality we use that $\widetilde \Lambda_{h, i} \preceq \Lambda_{0,h}$.
	The proof is complete by a union bound over $h$.
\end{proof}

\section{Auxiliary Lemmas}

\begin{lemma}[Extended value difference; \citealp{shani2020optimistic}, Lemma 1; \citealp{cai2020provably}]
\label{lem:extended_value_diff}
	Let $M = (H, \cS, \cA, \P, \l)$ be any MDP and $\pi, \pi' \in \cS \to \Delta_\cA$ be any two policies.
	Then, for any sequence of functions $\widehat Q_h^\pi \colon \cS \times \cA \to \R, V_h^\pi \colon \cS \to \R$, where $\widehat V_h^\pi(s) \eqq \abr{\pi(\cdot | s), \widehat Q_h(s, \cdot)}$,  $h = 1, \ldots, H$, we have
\begin{align*}
	\widehat V_1^{\pi} - V_1^{\pi'}
	&= 
	\sum_{h=1}^H \E_{s_h, a_h \sim d_h^{\pi'}}\sbr{
		\abr{\widehat Q^\pi_h(s_h, \cdot), 
		\pi_h(\cdot|s_h) - \pi'_h(\cdot|s_h)
		}
	}
	\\
	&\quad+
	\sum_{h=1}^H \E_{s_h, a_h \sim d_h^{\pi'}}\sbr{
		\widehat Q^\pi_h(s_h, a_h) - \l_h(s_h, a_h)
			- \P \widehat V_{h+1}^{\pi}(s_h, a_h)
	}.
\end{align*}	
\end{lemma}

\begin{lemma}[Covering number of Euclidean Ball]
	\label{lem:covering_l2_ball}
		For any $\epsilon > 0$, the $\epsilon$-covering of the Euclidean ball in $\R^d$ with radius $R > 0$ is upper bounded by $(1+2R/\epsilon)^d$.
	\end{lemma}

\begin{lemma}
	\label{lem:beta_recursion_choice}
		 Let $R, z \geq 1$, and $x \geq 2 z \log (R z)$. Then
		$z \log (R x) \leq x$.
	\end{lemma}
	\begin{proof}
		If $x = 2 z \log (R z)$;
		\begin{align*}
			  z \log (R x)
			&=z\log R + z\log (2 z \log (R z))
			\\
			&=z\log R + z\log (2 z) +z\log  \log (R z)
			\\
			&\leq z \log R + z \log z + z \log (R z)
			\\
			&= 2 z \log R + 2 z \log z 
			\\
			&= x.
		\end{align*}
		For larger values, the result follows by noting 
		$x- z\sqrt {\log (R x)}  $ is monotonically increasing in $x$ for all $x \geq z$.
	\end{proof}

\begin{lemma}[Lemma D.4 in \citealp{rosenberg2020near}]
\label{lem:nameless_concentration}
	Let $(\cF_k)_{k=1}^\infty$ be a filtration, and let $(X_k)_{k=1}^\infty$ be a sequence of random variables that are $\cF_k$-measurable, and supported on $[0, B]$. Then with probability $\geq 1-\delta$, we have that for any $K \geq 1$;
	\begin{align*}
		\sum_{k=1}^K \E\sbr{X_k \mid \cF_{k-1}}
		\leq 2 \sum_{k=1}^K X_k + 4 B \log \frac{2 K}{\delta}.
	\end{align*}
\end{lemma}

\begin{lemma}[Elliptical potential lemma, \citealp{lattimore2020bandit}, Lemma 19.4]
\label{lem:eliptical_potential}
	Let $(\phi_i)_{k=1}^K \subset \R^d$ with $\norm{\phi_k} \leq 1$, and set $\Lambda_k \eqq \lambda I + \sum_{i=1}^{k-1} \phi_i\phi_i\T  $ where $\lambda \geq 1$. Then,
	\begin{align*}
		\sum_{k=1}^K \norm{\phi_i}_{\Lambda_k^{-1}}^2 
		\leq 
		2 d\log \br{ 1+ \frac{K}{d \lambda } }
	\end{align*}
\end{lemma}
\begin{proof}
	Note that $\lambda \geq 1$ implies 
	$\norm{\phi_i}_{\Lambda_k^{-1}}^2 
		\leq \lambda_{\rm max}(\Lambda_k^{-1}) \norm{\phi_i}^2
		\leq \lambda^{-1} \leq 1$.
	Thus 
	\begin{align*}
		\sum_{k=1}^K \norm{\phi_i}_{\Lambda_k^{-1}}^2 
		= \sum_{k=1}^K \min \cbr{1, \norm{\phi_i}_{\Lambda_k^{-1}}^2 }.
	\end{align*}
	The rest of the proof is identical to \citet{lattimore2020bandit}, with $L=1$ and $V_0 = \lambda I$.
\end{proof}

\begin{lemma}[\citealp{cohen2019learning}, Lemma 27]
\label{lem:bonus_doubling_base}
	For any two matrices $A, B \in \R^{d \times d}$ which satisfy $0 \preceq A \preceq B$, we have $B \preceq \frac{\det B}{\det A} A$.
\end{lemma}

The following lemma is a direct consequence of the concentration of Self-Normalized Processes due to \citet{abbasi2011improved}.
\begin{lemma}
	\label{lem:ols_concentration}
	Let  $k \in \N$ and let $\l\colon \R^d \to \R$ denote a linear function $\l(\phi) = \phi\T g^\star$, $g^\star \in \R^d$.
    Assume $\cbr{\cF_i}_{i=1}^k$ is a filtration, and that $\phi_i \in \cF_{i-1}$ is an $\R^d$ valued stochastic process with $\norm{\phi_i} \leq 1$.
    Further, assume $\l^i = \l(\phi_i) + \xi_i$ where $\xi_i$ is a random variable such that $\E[ \xi_i \mid \cF_{i-1}] = 0$, and  $\av{\l^i} \leq D$ almost surely.
	Then for any $\delta > 0$, w.p. $1-\delta$, we have
	\begin{align*}
		\norm{\sum_{\tau=1}^k \phi_\tau 
			\br[B]{\l^\tau - \l(\phi_\tau)}}_{\Lambda_k^{-1}}^2
		\leq 2 D^2 d \log\br{\frac{k + \lambda}{\lambda}}, 
	\end{align*}
 where $\Lambda_k = \lambda I + \sum_{i=1}^k \phi_i \phi_i\T$
\end{lemma}

The next lemma establishes the \emph{uniform} concentration of least squares solutions over a \emph{class} of functions, and follows from a standard covering argument combined with the concentration of Self-Normalized Processes \cite{abbasi2011improved}.
% \begin{lemma}[\textbf{Uniform OLS error concentration}, \citealp{jin2020provably}, Lemma D.4]
% 	\label{lem:ols_unif_concentration}
% 	Let  $k \in \N$ and $\mathcal V$ denote a class of functions $V \colon \cS \to \R$ with $\norm{V}_\infty \leq D$.
% 	Further, assume $\phi_\tau \in \cF_{\tau-1}$ with $\norm{\phi_\tau} \leq 1$, and let
% 		$\Lambda_n = \lambda I + \sum_{i=1}^n \phi_i \phi_i\T$.
% 	Then for any $\delta > 0$, w.p. $1-\delta$, for all $V \in \mathcal V $ we have
% 	\begin{align*}
% 		\norm{\sum_{\tau=1}^k \phi_\tau 
% 			\br[B]{V(x_\tau) - \E \sbr{ V(x_\tau)|\cF_{\tau-1} }}}_{\Lambda_k^{-1}}^2
% 		\leq 4 D^2\br{
% 			\frac{d}{2} \log\br{\frac{k + \lambda}{\lambda}} 
% 			+ \log \frac{\mathcal N_\epsilon(\mathcal V)}{\delta}
% 		}	
% 		+ \frac{8 k^2 \epsilon^2}{\lambda},
% 	\end{align*}
% 	where $\mathcal N_\epsilon(\mathcal V) $ is $\norm{\cdot}_\infty$ covering number of $\mathcal V$.
% \end{lemma}

\begin{lemma}[\textbf{OLS uniform concentration}; \citealp{jin2020provably}, Lemma D.4]
	\label{lem:ols_unif_concentration}
 Let $\cbr{\cF_\tau}_{\tau=1}^\infty$ be a filtration. 
 Let $\cbr{x_\tau}$ be a stochastic process on state space $\cS$ that is $\cF_{\tau}$-measurable,
	and $\cbr{\phi_\tau}$ be an $\R^d$-valued stochastic process that is $\cF_{\tau-1}$-measurable and satisfies $\norm{\phi_\tau} \leq 1$.
	Further, let $\Lambda_n = \lambda I + \sum_{\tau=1}^n \phi_\tau \phi_\tau\T$.
Then for any $\delta > 0$, with probability at least $1-\delta$, for all $n\geq 1$ and any $V \in \mathcal V$ so that $\norm{V}_\infty \leq D$,
we have;
	\begin{align*}
		\norm{\sum_{\tau=1}^n \phi_\tau 
			\br[B]{V(x_\tau) - \E \sbr{ V(x_\tau)|\cF_{\tau-1} }}}_{\Lambda_n^{-1}}^2
		\leq 4 D^2\br{
			\frac{d}{2} \log\br{\frac{n + \lambda}{\lambda}} 
			+ \log \frac{\cN_\epsilon(\cV)}{\delta}
		}	
		+ \frac{8 n^2 \epsilon^2}{\lambda},
	\end{align*}
	where $\mathcal N_\epsilon(\mathcal V) $ is the $\norm{\cdot}_\infty$ $\epsilon$-covering number of $\cV$.
\end{lemma}

The next lemma is standard, for proof see e.g., \citet{hazan2016introduction,lattimore2020bandit}.
\begin{lemma}[\textbf{Entropy regularized OMD}]
\label{lem:omd}
	Let $y_1, \ldots, y_T \in \R^A$ be any sequence of vectors, and $\eta > 0$ such that $\eta y_t(a) \geq -1$ for all $t\in [T], a\in [A]$.
	Then if $\cbr{x_t}\subset \Delta_A$ is given by $x_1(a) = 1/n \forall a$, and for $t \geq 1$:
	\begin{align*}
		x_{t+1}(a) &= \frac{x_t(a)e^{-\eta y_t(a)}}{\sum_{a'\in [A]} x_t(a')e^{-\eta y_t(a')}}
        ,
	\end{align*}
	then,
	\begin{align*}
		\max_{x\in \Delta_A} \cbr{ \sum_{t=1}^T \abr{y_t, x_t - x} }
		\leq 
		\frac{\log A}{\eta } 
		+ \eta \sum_{k=1}^K \sum_{a=1}^A x_t(a)y_t(a)^2
		.
	\end{align*}
\end{lemma}

\end{document}